\documentclass{article}

\usepackage{times}
\usepackage{graphicx}
\usepackage{subfigure} 
\usepackage{amsmath,amssymb}
\usepackage{mathtools}
\usepackage{natbib}
\usepackage{algorithm}
\usepackage{algorithmic}
\usepackage[draft]{hyperref}
\usepackage{enumerate}
\usepackage{amsthm}
\usepackage{multirow}
\usepackage{adjustbox}
\usepackage{siunitx}
\usepackage[update,prepend]{epstopdf}

\newtheorem{theorem}{Theorem}

\usepackage[accepted]{icml2017}

\icmltitlerunning{Combining Model-Based and Model-Free Updates for Trajectory-Centric Reinforcement Learning}

\begin{document} 

\twocolumn[
\icmltitle{Combining Model-Based and Model-Free Updates for Trajectory-Centric\\Reinforcement Learning}

\icmlsetsymbol{equal}{*}

\begin{icmlauthorlist}
\icmlauthor{Yevgen Chebotar\textsuperscript{*}}{usc,mpi}
\icmlauthor{Karol Hausman\textsuperscript{*}}{usc}
\icmlauthor{Marvin Zhang\textsuperscript{*}}{berkeley}
\icmlauthor{Gaurav Sukhatme}{usc}
\icmlauthor{Stefan Schaal}{usc,mpi}
\icmlauthor{Sergey Levine}{berkeley} 
\end{icmlauthorlist}

\icmlaffiliation{usc}{University of Southern California, Los Angeles, CA, USA}

\icmlaffiliation{mpi}{Max Planck Institute for Intelligent Systems, T\"{u}bingen, Germany}

\icmlaffiliation{berkeley}{University of California Berkeley, Berkeley, CA, USA}

\icmlcorrespondingauthor{Yevgen Chebotar}{ychebota@usc.edu}

\icmlkeywords{}

\vskip 0.3in
]

\printAffiliationsAndNotice{\icmlEqualContribution}

\begin{abstract}
Reinforcement learning algorithms for real-world robotic applications must be able to handle complex, unknown dynamical systems while maintaining data-efficient learning. These requirements are handled well by model-free and model-based RL approaches, respectively.
In this work, we aim to combine the advantages of these approaches. By focusing on time-varying linear-Gaussian policies, we enable a model-based algorithm based on the linear-quadratic regulator that can be integrated into the model-free framework of path integral policy improvement.
We can further combine our method with guided policy search to train arbitrary parameterized policies such as deep neural networks.
Our simulation and real-world experiments demonstrate that this method can solve challenging manipulation tasks with comparable or better performance than model-free methods while maintaining the sample efficiency of model-based methods.
\end{abstract}
\vspace{-0.5cm}

\section{Introduction}
\label{sec:intro}

Reinforcement learning (RL) aims to enable automatic acquisition of behavioral skills, which can be crucial for robots and other autonomous systems to behave intelligently in unstructured real-world environments. However, real-world applications of RL have to contend with two often opposing requirements: data-efficient learning and the ability to handle complex, unknown dynamical systems that might be difficult to model. Real-world physical systems, such as robots, are typically costly and time consuming to run, making it highly desirable to learn using the lowest possible number of real-world trials. Model-based methods tend to excel at this~\cite{policysearch}, but suffer from significant bias, since complex unknown dynamics cannot always be modeled accurately enough to produce effective policies. Model-free methods have the advantage of handling arbitrary dynamical systems with minimal bias, but tend to be substantially less sample-efficient~\cite{kbp-rlrs-13,slmja-trpo-15}. Can we combine the efficiency of model-based algorithms with the final performance of model-free algorithms in a method that we can practically use on real-world physical systems?

\begin{figure}
\centering
\setlength{\tabcolsep}{2pt}
\begin{tabular}{r l}
\begin{adjustbox}{valign=t}
  \includegraphics[width=0.52\columnwidth]{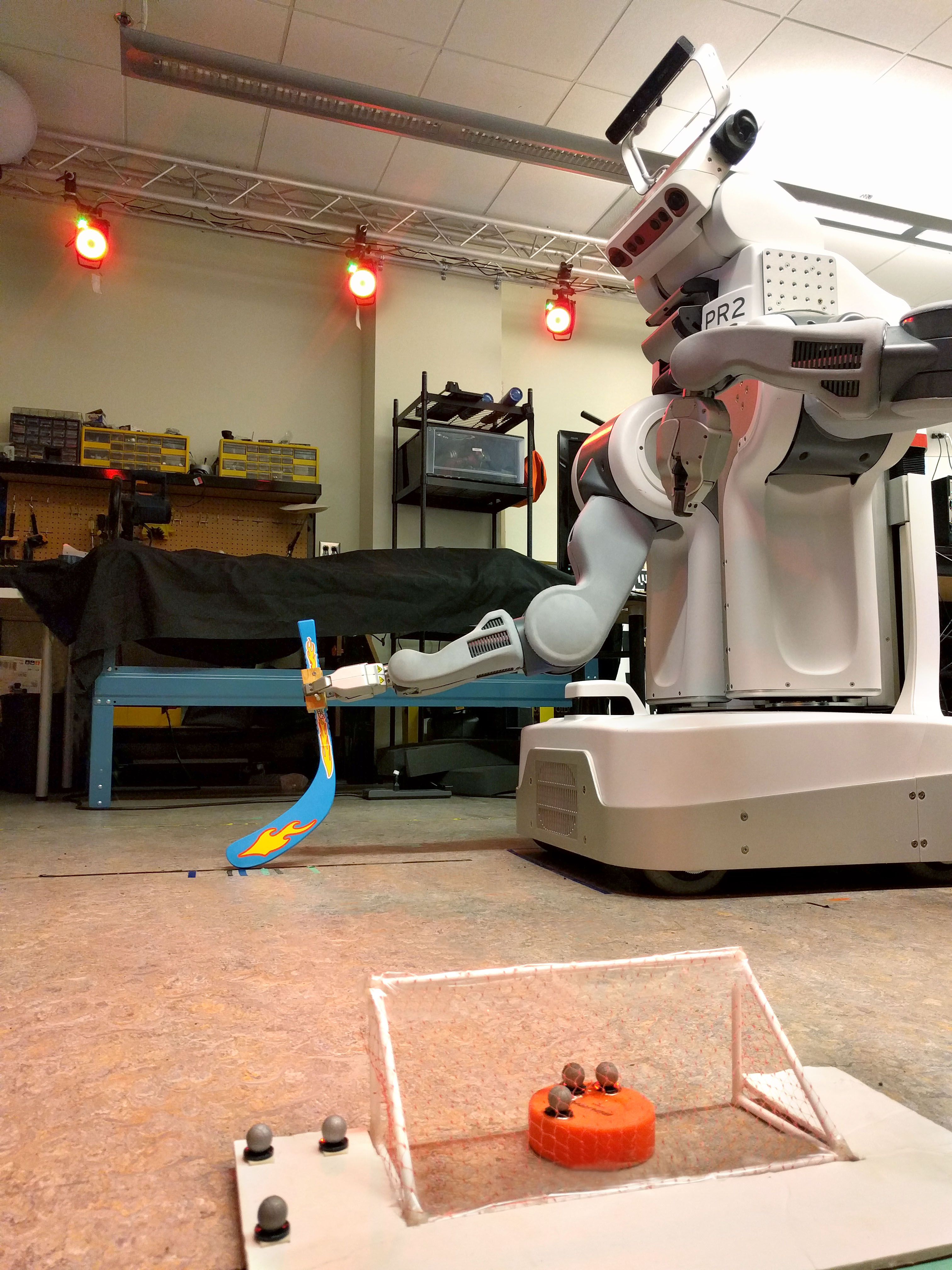}
\end{adjustbox}
&
\begin{adjustbox}{valign=t}
\begin{tabular}{@{}c@{}}
  \includegraphics[width=0.451\columnwidth]{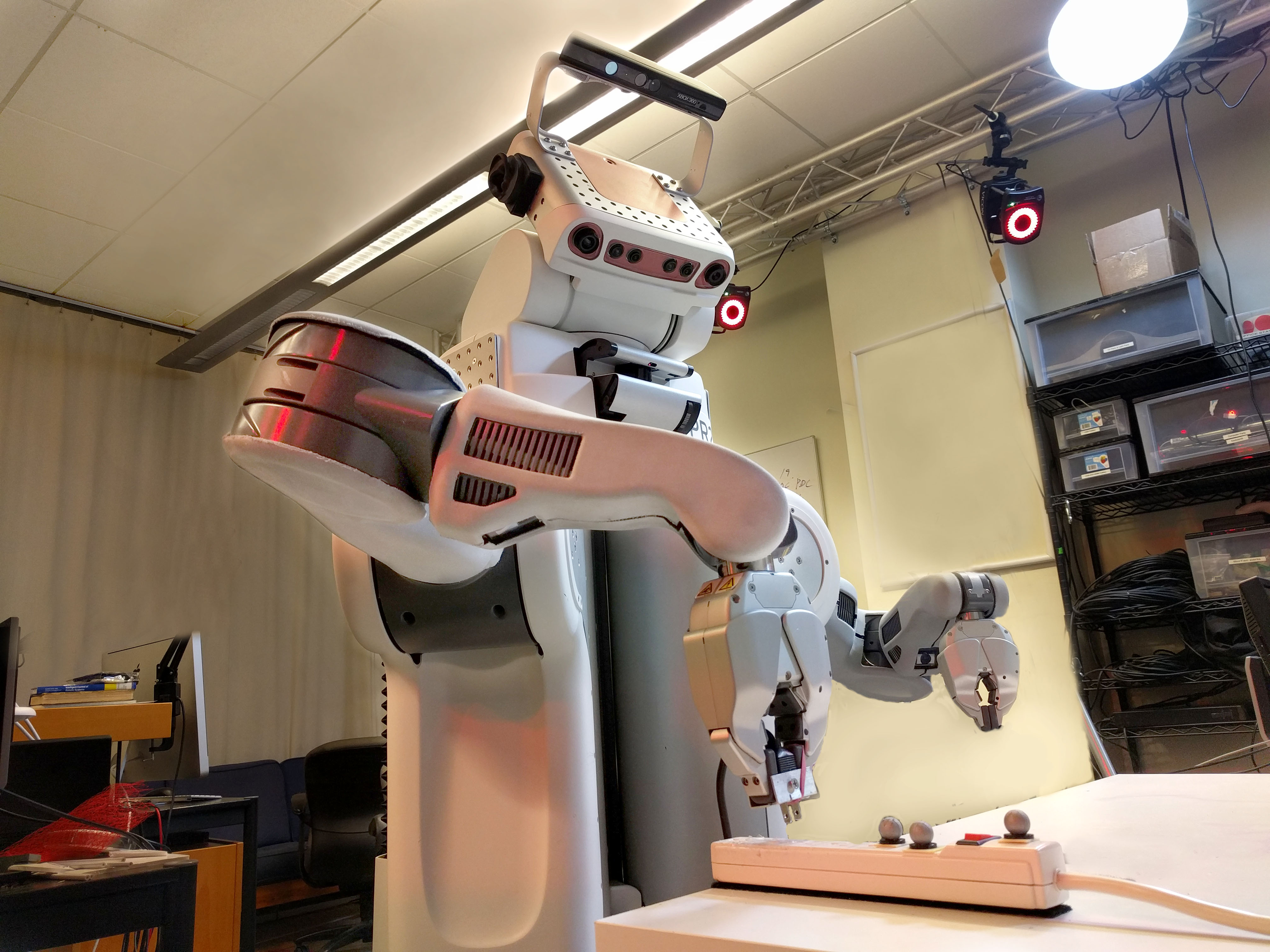}
  \\[0.0ex]
  \includegraphics[width=0.451\columnwidth]{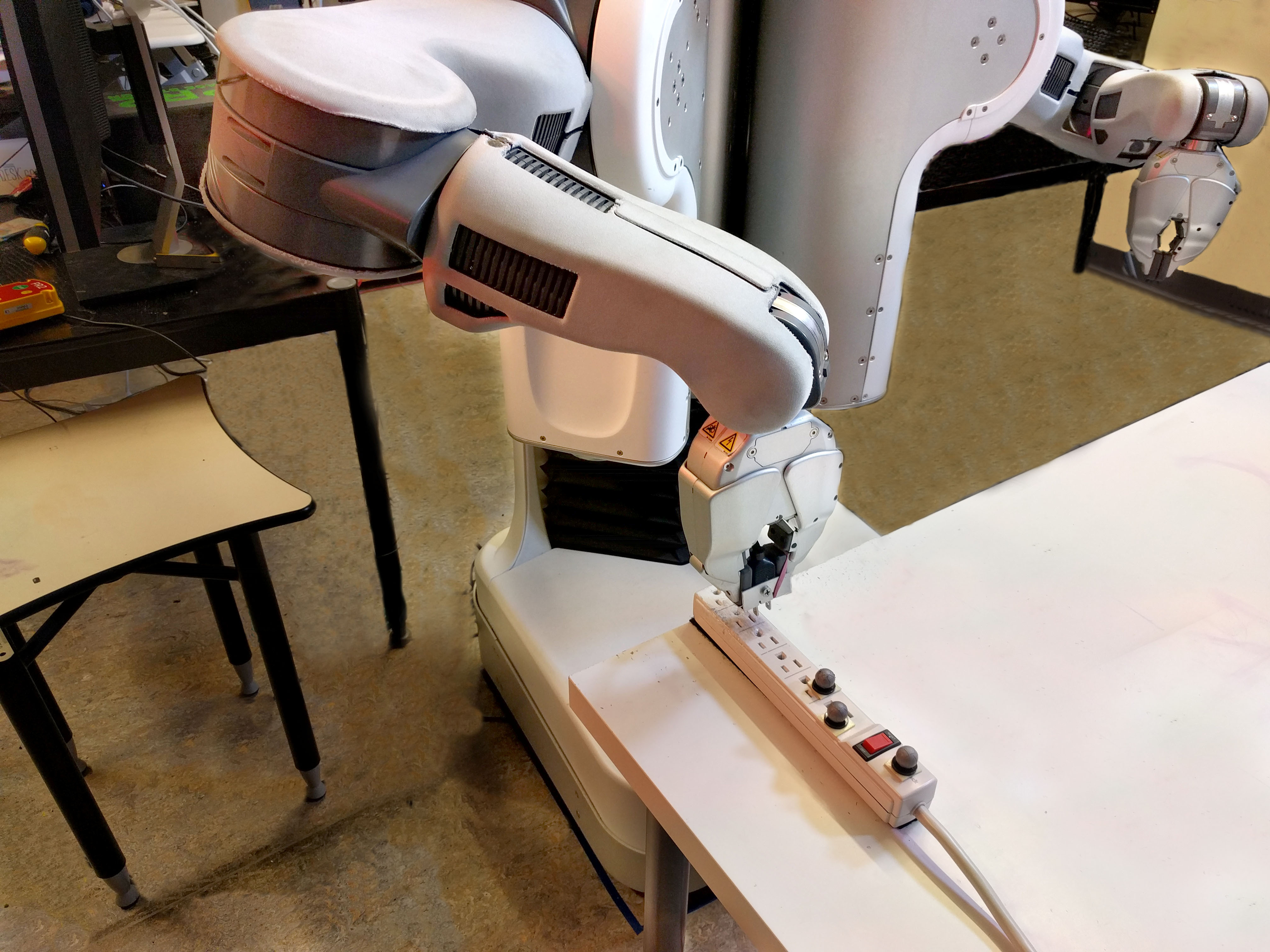}
  \end{tabular}
\end{adjustbox}
\end{tabular}
\vspace{-10pt}
\caption{Real robot tasks used to evaluate our method. Left: The hockey task which involves discontinuous dynamics. Right: The power plug task which requires high level of precision. Both of these tasks are learned from scratch without demonstrations.}
\label{fig:cover}
\vspace{-1pt}
\end{figure}

As we will discuss in Section~\ref{sec:related}, many prior methods that combine model-free and model-based techniques achieve only modest gains in efficiency or performance~\cite{stochastic_value_gradients,GuLSL16}. In this work, we aim to develop a method in the context of a specific policy representation: time-varying linear-Gaussian controllers. The structure of these policies provides us with an effective option for model-based updates via iterative linear-Gaussian dynamics fitting~\cite{LevineA14}, as well as a simple option for model-free updates via the path integral policy improvement (PI$^2$) algorithm~\cite{TheodorouBS10}.

Although time-varying linear-Gaussian (TVLG) policies are not as powerful as representations such as deep neural networks~\cite{mnih_et_al_atari,lhphe-ccdrl-16} or RBF networks~\cite{DeisenrothRF11}, they can represent arbitrary trajectories in continuous state-action spaces. Furthermore, prior work on guided policy search (GPS) has shown that TVLG policies can be used to train general-purpose parameterized policies, including deep neural network policies, for tasks involving complex sensory inputs such as vision~\cite{LevineA14,Levine:2016}. This yields a general-purpose RL procedure with favorable stability and sample complexity compared to fully model-free deep RL methods~\cite{montgomery_ajay_icra_paper}.

The main contribution of this paper is a procedure for optimizing TVLG policies that integrates both fast model-based updates via iterative linear-Gaussian model fitting and corrective model-free updates via the PI$^2$ framework. The resulting algorithm, which we call PILQR, combines the efficiency of model-based learning with the generality of model-free updates and can solve complex continuous control tasks that are infeasible for either linear-Gaussian models or PI$^2$ by itself, while remaining orders of magnitude more efficient than standard model-free RL. We integrate this approach into GPS to train deep neural network policies and present results both in simulation and on a real robotic platform. Our real-world results demonstrate that our method can learn complex tasks, such as hockey and power plug plugging (see Figure~\ref{fig:cover}), each with less than an hour of experience and no user-provided demonstrations.

\section{Related Work}
\label{sec:related}
The choice of policy representation has often been a crucial component in the success of a RL procedure~\cite{policysearch,kbp-rlrs-13}. 
Trajectory-centric representations, such as splines~\cite{peters_schaal_2008_reinforcement_learning_pg}, dynamic movement primitives~\cite{rl_dmps}, and TVLG controllers~\cite{peters_linear_gaussian_controller_paper,LevineA14} have proven particularly popular in robotics, where they can be used to represent cyclical and episodic motions and are amenable to a range of efficient optimization algorithms.
In this work, we build on prior work in trajectory-centric RL to devise an algorithm that is both sample-efficient and able to handle a wide class of tasks, all while not requiring human demonstration initialization.

More general representations for policies, such as deep neural networks, have grown in popularity recently due to their ability to process complex sensory input \cite{mnih_et_al_atari,lhphe-ccdrl-16,Levine:2016} and represent more complex strategies that can succeed from a variety of initial conditions~\cite{slmja-trpo-15,trpo-gae}. While trajectory-centric representations are more limited in their representational power, they can be used as an intermediate step toward efficient training of general parameterized policies using the GPS framework~\cite{Levine:2016}. Our proposed trajectory-centric RL method can also be combined with GPS to supervise the training of complex neural network policies. Our experiments demonstrate that this approach is several orders of magnitude more sample-efficient than direct model-free deep RL algorithms.

Prior algorithms for optimizing trajectory-centric policies can be categorized as model-free methods~\cite{TheodorouBS10,PetersMA10}, methods that use global models~\cite{pilco,pddp}, and methods that use local models~\cite{LevineA14,peters_linear_gaussian_controller_paper,peters_quadratic_models_paper}. Model-based methods typically have the advantage of being fast and sample-efficient, at the cost of making simplifying assumptions about the problem structure such as smooth, locally linearizable dynamics or continuous cost functions.
Model-free algorithms avoid these issues by not modeling the environment explicitly and instead improving the policy directly based on the returns, but this often comes at a cost in sample efficiency. Furthermore, many of the most popular model-free algorithms for trajectory-centric policies use example demonstrations to initialize the policies, since model-free methods require a large number of samples to make large, global changes to the behavior~\cite{TheodorouBS10,PetersMA10,peter_pastor_demonstration}.

Prior work has sought to combine model-based and model-free learning in several ways. \citet{Farshidianetal} also use LQR and PI$^2$, but do not combine these methods directly into one algorithm, instead using LQR to produce a good initialization for PI$^2$. Their work assumes the existence of a known model, while our method uses estimated local models. A number of prior methods have also looked at incorporating models to generate additional synthetic samples for model-free learning~\cite{dyna-q,GuLSL16}, as well as using models for improving the accuracy of model-free value function backups~\cite{stochastic_value_gradients}. Our work directly combines model-based and model-free updates into a single trajectory-centric RL method without using synthetic samples that degrade with modeling errors.

\section{Preliminaries}
The goal of policy search methods is to optimize the parameters $\theta$ of a policy $p(\mathbf{u}_t | \mathbf{x}_t)$, which defines a probability distribution over 
actions $\mathbf{u}_t$ conditioned on the system state $\mathbf{x}_t$ at each time step $t$ of a task execution. Let $\tau = (\mathbf{x}_1, \mathbf{u}_1, \dots,  \mathbf{x}_T, \mathbf{u}_T)$ be a trajectory of states and actions. Given a cost function $c(\mathbf{x}_t, \mathbf{u}_t)$, we define the trajectory cost as $c(\tau)=\sum_{t=1}^T c(\mathbf{x}_t, \mathbf{u}_t)$. The policy is optimized with respect to the expected cost of the policy
\vspace{-0.3cm}
\[
J(\theta) = \mathbb{E}_{p}\left[c(\tau)\right] = \int c(\tau) p (\tau) d\tau\,,
\vspace{-5pt}
\]
where $p(\tau)$ is the policy trajectory distribution given the system dynamics $p\left(\mathbf{x}_{t+1} | \mathbf{x}_{t}, \mathbf{u}_{t}\right)$
\vspace{-6pt}
\[
p (\tau) = p(\mathbf{x}_1) \prod_{t=1}^{T} p\left(\mathbf{x}_{t+1} | \mathbf{x}_{t}, \mathbf{u}_{t}\right) p(\mathbf{u}_t | \mathbf{x}_t)\,.
\vspace{-4pt}
\]
One policy class that allows us to employ an efficient model-based update is the TVLG controller $p(\mathbf{u}_t | \mathbf{x}_t) = \mathcal{N}(\mathbf{K}_{t} \mathbf{x}_t + \mathbf{k}_{t}, \mathbf{\Sigma}_{t})$.
In this section, we present the model-based and model-free algorithms that form the constituent parts of our hybrid method. The model-based method is an extension of a KL-constrained LQR algorithm~\cite{LevineA14}, which we shall refer to as LQR with fitted linear models (LQR-FLM). The model-free method is a PI$^2$ algorithm with per-time step KL-divergence constraints that is derived in previous work~\cite{chebotar-icra2017}.

\subsection{Model-Based Optimization of TVLG Policies}
\label{sec:ilqr}

The model-based method we use is based on the iterative linear-quadratic regulator (iLQR) and builds on prior work~\cite{LevineA14,synthesis}. We provide a full description and derivation in Appendix~\ref{app:lqr_flm}.

We use samples to fit a TVLG dynamics model $p(\mathbf{x}_{t+1}|\mathbf{x}_t,\mathbf{u}_t)=\mathcal{N}(\mathbf{f}_{\mathbf{x},t}\mathbf{x}_t+\mathbf{f}_{\mathbf{u},t}\mathbf{u}_t,\mathbf{F}_t)$ and assume a twice-differentiable cost function. As in \citet{synthesis}, we can compute a second-order Taylor approximation of our Q-function and optimize this with respect to $\mathbf{u}_t$ to find the optimal action at each time step $t$.
To deal with unknown dynamics, \citet{LevineA14} impose a KL-divergence constraint between the updated policy $p^{(i)}$ and previous policy $p^{(i-1)}$ to stay within the space of trajectories where the dynamics model is approximately correct. We similarly set up our optimization as

\vspace{-15pt}
{\small
\begin{align}
    \!\!\min_{p^{(i)}}~~\mathbb{E}_{p^{(i)}}[Q(\mathbf{x}_t,\mathbf{u}_t)]\,
    s.t.~\mathbb{E}_{p^{(i)}}\left[D_{\text{KL}}(p^{(i)}\|p^{(i-1)})\right]\leq\epsilon_t\,.\label{eq:lqr_optim}
\end{align}
}The main difference from \citet{LevineA14}
is that we enforce separate KL constraints for each linear-Gaussian policy rather than a single constraint on the induced trajectory distribution (i.e., compare Eq.~(\ref{eq:lqr_optim}) to the first equation in Section~3.1 of~\citet{LevineA14}).

LQR-FLM has substantial efficiency benefits over model-free algorithms. However, as our experimental results in Section~\ref{sec:experiments} show, the performance of LQR-FLM is highly dependent on being able to model the system dynamics accurately, causing it to fail for more challenging tasks.

\subsection{Policy Improvement with Path Integrals}
\label{pi2}

PI$^2$ is a model-free RL algorithm based on stochastic optimal control.
A detailed derivation of this method can be found in~\citet{TheodorouBS10}.

Each iteration of PI$^2$ involves generating $N$ trajectories by running the current policy. Let $S(\mathbf{x}_{i,t},\mathbf{u}_{i,t}) = c(\mathbf{x}_{i,t}, \mathbf{u}_{i,t}) + \sum^T_{j=t+1} c(\mathbf{x}_{i,j}, \mathbf{u}_{i,j})$ be the cost-to-go of trajectory $i \in \{1,\ldots,N\}$ starting in state $\mathbf{x}_{i,t}$ by performing action $\mathbf{u}_{i,t}$ and following the policy $p(\mathbf{u}_t | \mathbf{x}_t)$ afterwards. Then, we can compute  probabilities $P(\mathbf{x}_{i,j}, \mathbf{u}_{i,j})$ for each trajectory starting at time step $t$

\vspace{-10pt}
{\small
\begin{equation}
P(\mathbf{x}_{i,t}, \mathbf{u}_{i,t}) = \frac{ \exp \left(-\frac{1}{\eta_t} S(\mathbf{x}_{i,t}, \mathbf{u}_{i,t}) \right)}{\int \exp\left(-\frac{1}{\eta_t} S(\mathbf{x}_{i,t}, \mathbf{u}_{i,t})\right) d \mathbf{u}_{i,t}}\,.\label{eqn:pi2update}
\end{equation}
}The probabilities follow from the Feynman-Kac theorem applied to stochastic optimal control~\cite{TheodorouBS10}.
The intuition is that the trajectories with lower costs receive higher probabilities, and the policy distribution shifts towards a lower cost trajectory region. The costs are scaled by $\eta_t$, which can be interpreted as the temperature of a soft-max distribution. This is similar to the dual variables $\eta_t$ in LQR-FLM in that they control the KL step size, however they are derived and computed differently. After computing the new probabilities $P$, we update the policy distribution by reweighting each sampled control $\mathbf{u}_{i,t}$ by $P(\mathbf{x}_{i,t}, \mathbf{u}_{i,t})$ and updating the policy parameters by a maximum likelihood estimate~\cite{chebotar-icra2017}.

To relate PI$^2$ updates to LQR-FLM optimization of a constrained objective, which is necessary for combining these methods, we can formulate the following theorem. 
\begin{theorem}
\label{theo:pi2}
The PI$^2$ update corresponds to a KL-constrained minimization of the expected cost-to-go $S(\mathbf{x}_{t},\mathbf{u}_{t}) = \sum^T_{j=t} c(\mathbf{x}_{j}, \mathbf{u}_{j})$ at each time step $t$
\vspace{-2pt}
\begin{align*}
\!\min_{p^{(i)}}~\mathbb{E}_{p^{(i)}}[S(\mathbf{x}_{t},\mathbf{u}_{t})]~
s.t.~\mathbb{E}_{p^{(i-1)}}\!\!\left[D_\text{KL} \left(p^{(i)}\|\, p^{(i-1)}\right)\right]\leq \epsilon\,,
\vspace{-7pt}
\end{align*}
where $\epsilon$ is the maximum KL-divergence between the new policy $p^{(i)}\left (  \mathbf{u}_t  | \mathbf{x}_t \right)$ and the old policy $p^{(i-1)}\left (  \mathbf{u}_t  | \mathbf{x}_t \right)$.
\end{theorem}
\begin{proof}
The Lagrangian of this problem is given by
\vspace{-1pt}
{\small
\begin{align*}
\mathcal{L}(p^{(i)}\!\!, \eta_t) \!=\! \mathbb{E}_{p^{(i)}}\![S(\mathbf{x}_{t},\mathbf{u}_{t})] \!+\! \eta_t \mathbb{E}_{p^{(i-1)}}\!\left[\! D_\text{KL} \!\left(p^{(i)}\|\, p^{(i-1)}\!\right)\! \!-\! \epsilon \right].
\end{align*}
}By minimizing the Lagrangian with respect to $p^{(i)}$ we can find its relationship to $p^{(i-1)}$ (see Appendix \ref{app:pi2_proof}), given by
{\small
\begin{align}
\!\!\!p^{(i)}\!\left(  \mathbf{u}_t  | \mathbf{x}_t \right) 
\!\propto\! p^{(i-1)}\!\left(\mathbf{u}_t| \mathbf{x}_t \right) \mathbb{E}_{p^{(i-1)}}\!\!\left[\exp \left(\!-\frac{1}{\eta_t} S(\mathbf{x}_{t},\mathbf{u}_{t}) \right)\!\right] \!. \!\!\label{pi2_update}
\end{align}
}This gives us an update rule for $p^{(i)}$ that corresponds exactly to reweighting the controls from the previous policy $p^{(i-1)}$ based on their probabilities $P(\mathbf{x}_{t},\mathbf{u}_{t})$ described earlier. The temperature $\eta_t$ now corresponds to the dual variable of the KL-divergence constraint.
\end{proof}
The temperature $\eta_t$ can be estimated at each time step separately by optimizing the dual function
\begin{align}
g(\eta_t) \!=\! \eta_t \epsilon \!+\! \eta_t \log \mathbb{E}_{p^{(i-1)}} \left [ \exp \left(-\frac{1}{\eta_t} S(\mathbf{x}_{t},\mathbf{u}_{t}) \right) \right ] \!, \!\label{eq:reps_dual}
\end{align}
with derivation following from~\citet{PetersMA10}.

PI$^2$ was used by \citet{chebotar-icra2017} to solve several challenging robotic tasks such as door opening and pick-and-place, where they achieved better final performance than LQR-FLM. However, due to its greater sample complexity, PI$^2$ required initialization from demonstrations.

\section{Integrating Model-Based Updates into PI$^2$}

Both PI$^2$ and LQR-FLM can be used to learn TVLG policies and both have their strengths and weaknesses. In this section, we first show how the PI$^2$ update can be broken up into two parts, with one part using a model-based cost approximation and another part using the residual cost error after this approximation. Next, we describe our method for integrating model-based updates into PI$^2$ by using our extension of LQR-FLM to optimize the linear-quadratic cost approximation and performing a subsequent update with PI$^2$ on the residual cost. We demonstrate in Section~\ref{sec:experiments} that our method combines the strengths of PI$^2$ and LQR-FLM while compensating for their weaknesses.

\subsection{Two-Stage PI$^2$ update}
\label{2stage_pi2}
To integrate a model-based optimization into PI$^2$, we can divide it into two steps. Given an approximation $\hat{c}(\mathbf{x}_t, \mathbf{u}_t)$ of the real cost $c(\mathbf{x}_t, \mathbf{u}_t)$ and the residual cost $\tilde{c}(\mathbf{x}_t, \mathbf{u}_t) = c(\mathbf{x}_t, \mathbf{u}_t) - \hat{c}(\mathbf{x}_t, \mathbf{u}_t)$,
let $\hat{S}_t = \hat{S}(\mathbf{x}_t, \mathbf{u}_t)$ be the approximated cost-to-go of a trajectory starting with state $\mathbf{x}_t$ and action $\mathbf{u}_t$,  and $\tilde{S}_t = \tilde{S}(\mathbf{x}_t, \mathbf{u}_t)$ be the residual of the real cost-to-go $S(\mathbf{x}_t, \mathbf{u}_t)$ after approximation. We can rewrite the PI$^2$ policy update rule from Eq.~(\ref{pi2_update}) as
\begin{align}
    p^{(i)}&\left (  \mathbf{u}_t  | \mathbf{x}_t \right) \nonumber \\
    &\propto p^{(i-1)}\left(\mathbf{u}_t| \mathbf{x}_t \right) \mathbb{E}_{p^{(i-1)}}\left[\exp \left(-\frac{1}{\eta_t} \left(\hat{S}_t + \tilde{S}_t\right) \right)\right] \nonumber \\
    &\propto \hat{p}\left(\mathbf{u}_t| \mathbf{x}_t \right) \mathbb{E}_{p^{(i-1)}}\left[\exp \left(-\frac{1}{\eta_t}  \tilde{S}_t \right)\right], \label{eq:pi2_res_update}
\end{align}
where $\hat{p}\left(\mathbf{u}_t| \mathbf{x}_t \right)$ is given by
\begin{align}
\hat{p}\left(\mathbf{u}_t| \mathbf{x}_t \right) \propto p^{(i-1)}\left(\mathbf{u}_t| \mathbf{x}_t \right) \mathbb{E}_{p^{(i-1)}}\left[\exp \left(-\frac{1}{\eta_t} \hat{S}_t  \right)\right]\!. \label{approx_update}
\end{align}
Hence, by decomposing the cost into its approximation and the residual approximation error, the PI$^2$ update can be split into two steps: (1) update using the approximated costs $\hat{c}(\mathbf{x}_t, \mathbf{u}_t)$ and samples from the old policy $p^{(i-1)}\left(\mathbf{u}_t| \mathbf{x}_t \right)$ to get $\hat{p}\left (  \mathbf{u}_t  | \mathbf{x}_t \right)$; (2) update $p^{(i)}\left (  \mathbf{u}_t  | \mathbf{x}_t \right)$ using the residual costs $\tilde{c}(\mathbf{x}_t, \mathbf{u}_t)$ and samples from $\hat{p}\left(\mathbf{u}_t| \mathbf{x}_t \right)$.

\subsection{Model-Based Substitution with LQR-FLM}

We can use Theorem~(\ref{theo:pi2}) to rewrite Eq.~(\ref{approx_update}) as a constrained optimization problem
\begin{align}
\min_{\hat{p}}~~\mathbb{E}_{\hat{p}}\left[\hat{S}(\mathbf{x}_t, \mathbf{u}_t)\right]\nonumber
s.t.~~\mathbb{E}_{p^{(i-1)}}\left[D_\text{KL} \left(\hat{p}\|\, p^{(i-1)}\right)\right]\leq \epsilon\,.
\end{align}
Thus, the policy $\hat{p}\left(\mathbf{u}_t| \mathbf{x}_t \right)$ can be updated using any algorithm that can solve this optimization problem. By choosing a model-based approach for this, we can speed up the learning process significantly.
Model-based methods are typically constrained to some particular cost approximation, however, PI$^2$ can accommodate any form of $\tilde{c}\left(\mathbf{x}_t, \mathbf{u}_t\right)$ and thus will handle arbitrary cost residuals.

LQR-FLM solves the type of constrained optimization problem in Eq.~(\ref{eq:lqr_optim}), which matches the optimization problem needed to obtain $\hat{p}$, where the cost-to-go $\hat{S}$ is approximated with a quadratic cost and a linear-Gaussian dynamics model.\footnote{In practice, we make a small modification to the problem in Eq.~(\ref{eq:lqr_optim}) so that the expectation in the constraint is evaluated with respect to the new distribution $\hat{p}(\mathbf{x}_t)$ rather than the previous one $p^{(i-1)}(\mathbf{x}_t)$. This modification is heuristic and no longer aligns with Theorem~(\ref{theo:pi2}), but works better in practice.}
We can thus use LQR-FLM to perform our first update, which enables greater efficiency but is susceptible to modeling errors when the fitted local dynamics are not accurate, such as in discontinuous systems. We can use a PI$^2$ optimization on the residuals to correct for this bias.

\subsection{Optimizing Cost Residuals with PI$^2$}

In order to perform a PI$^2$ update on the residual costs-to-go $\tilde{S}$, we need to know what $\hat{S}$ is for each sampled trajectory. That is, what is the cost-to-go that is actually used by LQR-FLM to make its update?
The structure of the algorithm implies a specific cost-to-go formulation for a given trajectory -- namely, the sum of quadratic costs obtained by running the same policy under the TVLG dynamics used by LQR-FLM. A given trajectory can be viewed as being generated by a deterministic policy conditioned on a particular noise realization $\xi_{i,1},\dots,\xi_{i,T}$, with actions given by
\begin{align}
\mathbf{u}_{i,t} = \mathbf{K}_t \mathbf{x}_{i,t} + \mathbf{k}_t + \sqrt{\mathbf{\Sigma}_t} \xi_{i,t}\,, \label{eq:reparam}
\end{align}
\noindent where $\mathbf{K}_t$, $\mathbf{k}_t$, and $\mathbf{\Sigma_t}$ are the parameters of $p^{(i-1)}$. We can therefore evaluate $\hat{S}(\mathbf{x}_t,\mathbf{u}_t)$ by simulating this deterministic controller from $(\mathbf{x}_t,\mathbf{u}_t)$ under the fitted TVLG dynamics and evaluating its time-varying quadratic cost, and then plugging these values into the residual cost.

In addition to the residual costs $\tilde{S}$ for each trajectory, the PI$^2$ update also requires control samples from the updated LQR-FLM policy $\hat{p}\left (\mathbf{u}_t | \mathbf{x}_t \right)$. Although we have the updated LQR-FLM policy, we only have samples from the old policy $p^{(i-1)}\left (\mathbf{u}_t |\mathbf{x}_t \right)$. However, we can apply a form of the re-parametrization trick~\cite{KingmaW13} and again use the stored noise realization of each trajectory $\xi_{t,i}$ to evaluate what the control would have been for that sample under the LQR-FLM policy $\hat{p}$. The expectation of the residual cost-to-go in Eq.~(\ref{eq:pi2_res_update}) is taken with respect to the old policy distribution $p^{(i-1)}$. Hence, we can reuse the states $\mathbf{x}_{i,t}$ and their corresponding noise $\xi_{i,t}$ that was sampled while rolling out the previous policy $p^{(i-1)}$ and evaluate the new controls according to $\mathbf{\hat{u}}_{i,t} = \mathbf{\hat{K}}_t \mathbf{x}_{i,t} + \mathbf{\hat{k}}_t + \sqrt{\mathbf{\hat{\Sigma}}_t} \xi_{i,t}$. This linear transformation on the sampled control provides unbiased samples from $\hat{p}(\mathbf{u}_t|\mathbf{x}_t)$. After transforming the control samples, they are reweighted according to their residual costs and plugged into the PI$^2$ update in Eq.~(\ref{eqn:pi2update}).

\subsection{Summary of PILQR algorithm}
\label{sec:step-adj}

Algorithm \ref{algo:pilqr} summarizes our method for combining LQR-FLM and PI$^2$ to create a hybrid model-based and model-free algorithm. 
After generating a set of trajectories by running the current policy (line 2), we fit TVLG dynamics and compute the quadratic cost approximation $\hat{c}(\mathbf{x}_t, \mathbf{u}_t)$ and approximation error residuals $\tilde{c}(\mathbf{x}_t, \mathbf{u}_t)$ (lines 3, 4). 
In order to improve the convergence behavior of our algorithm, we adjust the KL-step $\epsilon_t$ of the LQR-FLM optimization in Eq.~(\ref{eq:lqr_optim}) based inversely on the proportion of the residual costs-to-go to the sampled costs-to-go (line 5). 
In particular, if the ratio between the residual and the overall cost is sufficiently small or large, we increase or decrease, respectively, the KL-step $\epsilon_t$. 
We then continue with optimizing for the temperature $\eta_t$ using the dual function from Eq.~(\ref{eq:reps_dual}) (line 6). Finally, we perform an LQR-FLM update on the cost approximation (line 7) and a subsequent PI$^2$ update using the cost residuals (line 8).
As PILQR combines LQR-FLM and PI$^2$ updates in sequence in each iteration, its computational complexity can be determined as the sum of both methods. 
Due to the properties of PI$^2$, the covariance of the optimized TVLG controllers decreases each iteration and the method eventually converges to a single solution.

\section{Training Parametric Policies with GPS}

\begin{figure}
    \centering
    \setlength{\unitlength}{0.5\columnwidth}
    \begin{picture}(1.0, 0.5)
        \put(-0.5, 0.0){\includegraphics[width=0.33\columnwidth]{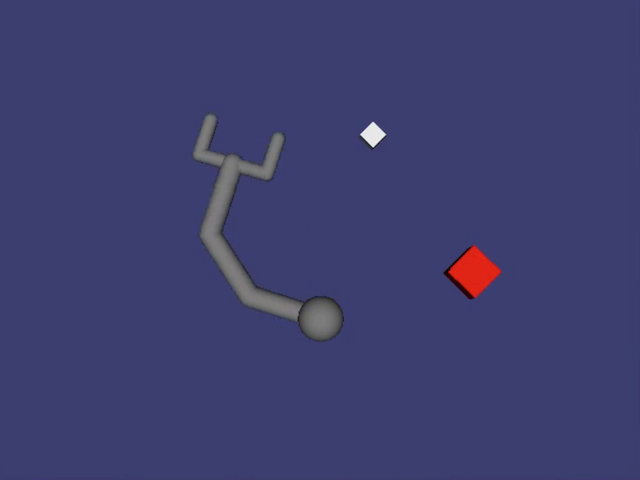}}
        \put(0.17, 0.0){\includegraphics[width=0.33\columnwidth]{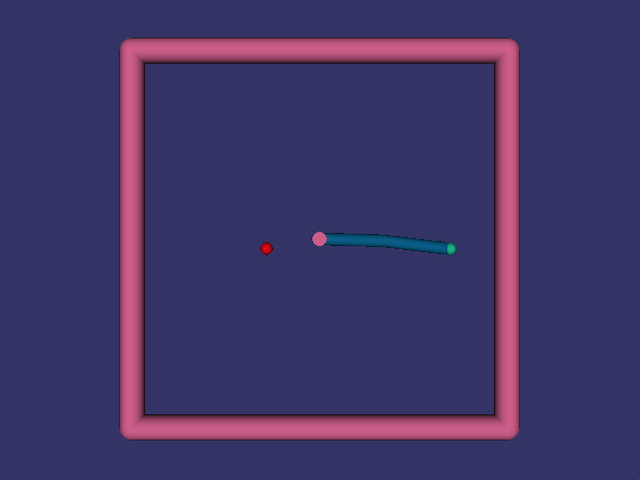}}
        \put(0.84, 0.0){\includegraphics[width=0.33\columnwidth]{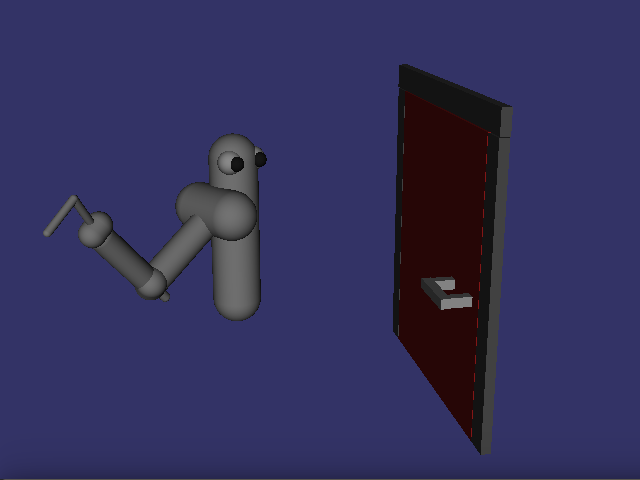}}
    \end{picture}
    \vspace{-5pt}
    \caption{We evaluate on a set of simulated robotic manipulation tasks with varying difficulty. Left to right, the tasks involve pushing a block, reaching for a target, and opening a door in 3D.}
    \label{fig:simtasks}
\end{figure}

PILQR offers an approach to perform trajectory optimization of TVLG policies. In this work, we employ mirror descent guided policy search (MDGPS) \cite{MontgomeryL16} in order to use PILQR to train parametric policies, such as neural networks. Instead of directly learning the parameters of a high-dimensional parametric or ``global policy'' with RL, we first learn simple TVLG policies, which we refer to as ``local policies'' $p(\mathbf{u}_t | \mathbf{x}_t)$ for various initial conditions of the task.
After optimizing the local policies, the optimized controls from these policies are used to create a training set for learning the global policy $\pi_\theta$ in a supervised manner. Hence, the final global policy generalizes across multiple local policies.

Using the TVLG representation of the local policies makes it straightforward to incorporate PILQR into the MDGPS framework. Instead of constraining against the old local TVLG policy as in Theorem~(\ref{theo:pi2}), each instance of the local policy is now constrained against the old global policy

\begin{algorithm}[t]
\begin{algorithmic}[1]
\FOR{iteration $k \in \{1,\dots,K\}$}
    \STATE{ Generate trajectories $\mathcal{D} = \{\tau_{i}\}$ by running the current linear-Gaussian policy $p^{(k-1)}\left (  \mathbf{u}_t  | \mathbf{x}_t \right)$}
    \STATE{Fit TVLG dynamics $\hat{p}\left(\mathbf{x}_{t+1} | \mathbf{x}_{t}, \mathbf{u}_{t}\right)$}
    \STATE{Estimate cost approximation $\hat{c}(\mathbf{x}_t, \mathbf{u}_t)$ using fitted dynamics and compute cost residuals: \\$\tilde{c}(\mathbf{x}_t, \mathbf{u}_t)  = c(\mathbf{x}_t, \mathbf{u}_t)  - \hat{c}(\mathbf{x}_t, \mathbf{u}_t) $}
    \STATE{Adjust LQR-FLM KL step $\epsilon_t$ based on ratio of residual costs-to-go $\tilde{S}$ and sampled costs-to-go $S$}
    \STATE{Compute $\eta_t$ using dual function from Eq.~(\ref{eq:reps_dual})}
    \STATE{Perform LQR-FLM update to compute $\hat{p}\left (  \mathbf{u}_t  | \mathbf{x}_t \right)$:
    $\min_{p^{(i)}}~\mathbb{E}_{p^{(i)}}[Q(\mathbf{x}_t,\mathbf{u}_t)]$ \\
    ~~~~~~$s.t.~~\mathbb{E}_{p^{(i)}}\left[D_{\text{KL}}(p^{(i)}\|p^{(i-1)})\right]\leq\epsilon_t$}
    \STATE{Perform PI$^2$ update using cost residuals and LQR-FLM actions to compute the new policy: \\
     $p^{(k)}\left (  \mathbf{u}_t  | \mathbf{x}_t \right) \propto \hat{p}\left(\mathbf{u}_t| \mathbf{x}_t \right) \mathbb{E}_{p^{(i-1)}}\left[\exp \left(-\frac{1}{\eta_t}  \tilde{S}_t) \right)\right]$}
\ENDFOR
\end{algorithmic}
\caption{PILQR algorithm}
\label{algo:pilqr}
\end{algorithm}

\vspace{-0.5cm}
{\small
\begin{align*}
\min_{p^{(i)}}~~\mathbb{E}_{p^{(i)}}[S(\mathbf{x}_{t},\mathbf{u}_{t})]
s.t.~~\mathbb{E}_{p^{(i-1)}}\left[D_\text{KL} \left(p^{(i)}\|\, \pi^{(i-1)}_\theta\right)\right]\leq \epsilon\,.
\vspace{-0.3cm}
\end{align*}
}The two-stage update proceeds as described in Section~\ref{2stage_pi2}, with the change that the LQR-FLM policy is now constrained against the old global policy $\pi^{(i-1)}_\theta$.

\section{Experimental Evaluation}
\label{sec:experiments}

Our experiments aim to answer the following questions: (1)~How does our method compare to other trajectory-centric and deep RL algorithms in terms of final performance and sample efficiency? (2)~Can we utilize linear-Gaussian policies trained using PILQR to obtain robust neural network policies using MDGPS? (3)~Is our proposed algorithm capable of learning complex manipulation skills on a real robotic platform? We study these questions through a set of simulated comparisons against prior methods, as well as real-world tasks using a PR2 robot. The performance of each method can be seen in our supplementary video.\footnote{\mbox{\url{https://sites.google.com/site/icml17pilqr}}} Our focus in this work is specifically on robotics tasks that involve manipulation of objects, since such tasks often exhibit elements of continuous and discontinuous dynamics and require sample-efficient methods, making them challenging for both model-based and model-free methods.

\subsection{Simulation Experiments}
\label{sec:simresults}

\begin{figure}
    \centering
    \includegraphics[width=0.99\columnwidth]{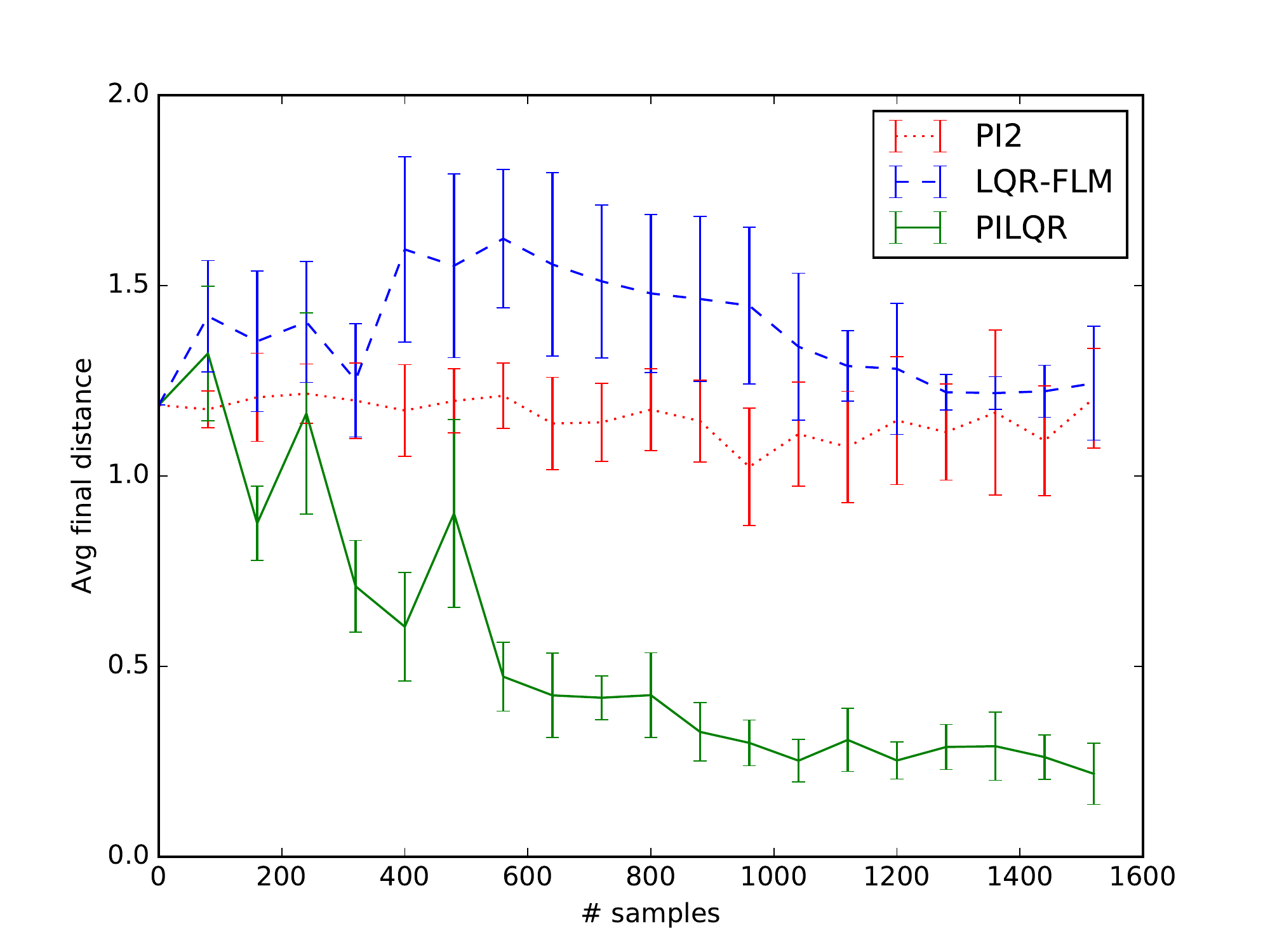}
    \vspace{-10pt}
    \caption{Average final distance from the block to the goal on one condition of the gripper pusher task. This condition is difficult due to the block being initialized far away from the gripper and the goal area, and only PILQR is able to succeed in reaching the block and pushing it toward the goal. Results for additional conditions are available in Appendix~\ref{app:sim_res}, and the supplementary video demonstrates the final behavior of each learned policy.}
    \label{fig:gripperpusher}
\end{figure}

We evaluate our method on three simulated robotic manipulation tasks, depicted in Figure~\ref{fig:simtasks} and discussed below:

\noindent {\bf Gripper pusher.} This task involves controlling a 4 DoF arm with a gripper to push a white block to a red goal area.
The cost function is a weighted combination of the distance from the gripper to the block and from the block to the goal.

\noindent {\bf Reacher.} The reacher task from OpenAI gym~\cite{gym} requires moving the end of a 2 DoF arm to a target position. This task is included to provide comparisons against prior methods. The cost function is the distance from the end effector to the target. We modify the cost function slightly: the original task uses an $\ell_2$ norm, while we use a differentiable Huber-style loss, which is more typical for LQR-based methods~\cite{synthesis}.

\noindent {\bf Door opening.} This task requires opening a door with a 6 DoF 3D arm. The arm must grasp the handle and pull the door to a target angle, which can be particularly challenging for model-based methods due to the complex contacts between the hand and the handle, and the fact that a contact must be established before the door can be opened. The cost function is a weighted combination of the distance of the end effector to the door handle and the angle of the door.

\noindent Additional experimental setup details, including the exact cost functions, are provided in Appendix~\ref{app:sim_setup}.

We first compare PILQR to LQR-FLM and PI$^2$ on the gripper pusher and door opening tasks.
Figure~\ref{fig:gripperpusher} details performance of each method on the most difficult condition for the gripper pusher task. Both LQR-FLM and PI$^2$ perform significantly worse on the two more difficult conditions of this task. While PI$^2$ improves in performance as we provide more samples, LQR-FLM is bounded by its ability to model the dynamics, and thus predict the costs, at the moment when the gripper makes contact with the block. Our method solves all four conditions with 400 total episodes per condition and, as shown in the supplementary video, is able to learn a diverse set of successful behaviors including flicking, guiding, and hitting the block. On the door opening task, PILQR trains TVLG policies that succeed at opening the door from each of the four initial robot positions. While the policies trained with LQR-FLM are able to reach the handle, they fail to open the door.

\begin{figure}
    \centering
    \includegraphics[width=0.99\columnwidth]{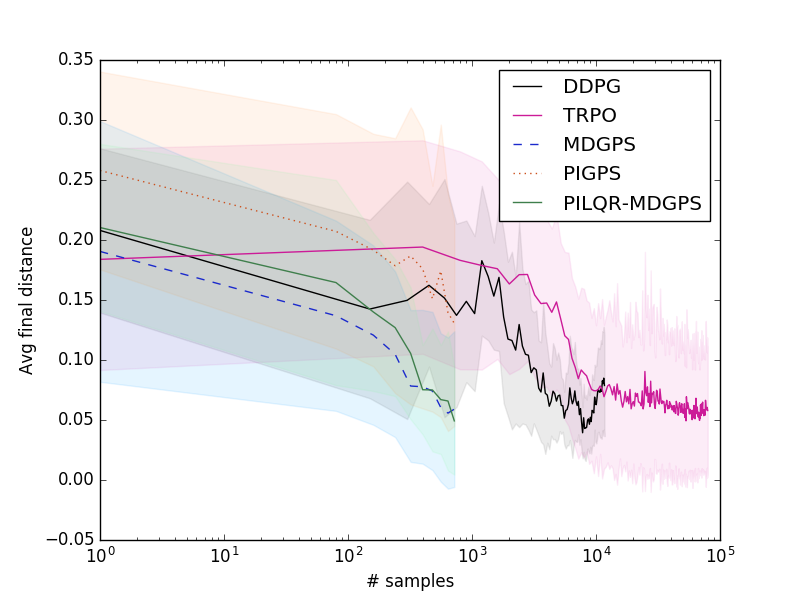}
    \vspace{-10pt}
    \caption{Final distance from the reacher end effector to the target averaged across 300 random test conditions per iteration. MDGPS with LQR-FLM, MDGPS with PILQR, TRPO, and DDPG all perform competitively. However, as the log scale for the x axis shows, TRPO and DDPG require orders of magnitude more samples. MDGPS with PI$^2$ performs noticeably worse.}
    \label{fig:reacher}
\end{figure}

Next we evaluate neural network policies on the reacher task. Figure~\ref{fig:reacher} shows results for MDGPS with each local policy method, as well as two prior deep RL methods that directly learn neural network policies: trust region policy optimization (TRPO)~\cite{slmja-trpo-15} and deep deterministic policy gradient (DDPG)~\cite{lhphe-ccdrl-16}. MDGPS with LQR-FLM and MDGPS with PILQR perform competitively in terms of the final distance from the end effector to the target, which is unsurprising given the simplicity of the task, whereas MDGPS with PI$^2$ is again not able to make much progress. On the reacher task, DDPG and TRPO use 25 and 150 times more samples, respectively, to achieve approximately the same performance as MDGPS with LQR-FLM and PILQR. For comparison, amongst previous deep RL algorithms that combined model-based and model-free methods, SVG and NAF with imagination rollouts reported using approximately up to five times fewer samples than DDPG on a similar reacher task~\cite{stochastic_value_gradients,GuLSL16}. Thus we can expect that MDGPS with our method is about one order of magnitude more sample-efficient than SVG and NAF. While this is a rough approximation, it demonstrates a significant improvement in efficiency.

\begin{figure}[t]
    \centering
    \includegraphics[width=0.99\columnwidth]{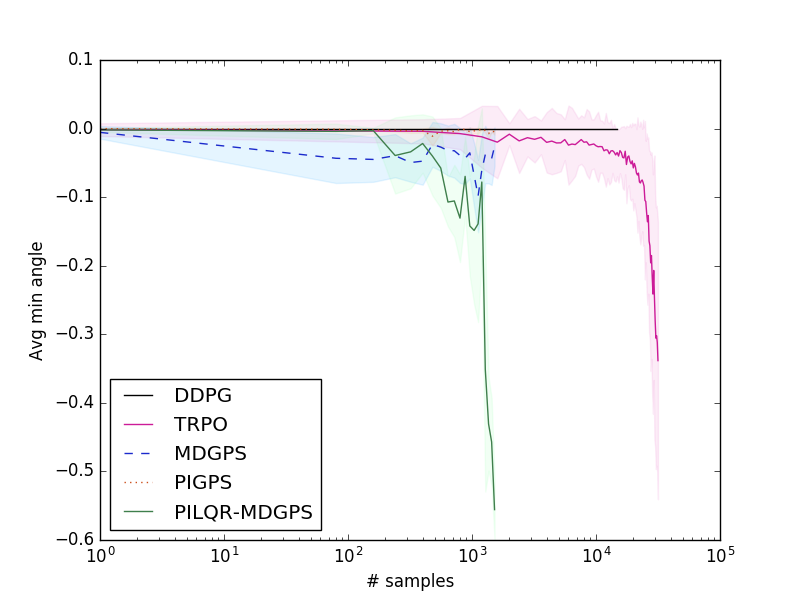}
    \vspace{-8pt}
    \caption{Minimum angle in radians of the door hinge (lower is better) averaged across 100 random test conditions per iteration. MDGPS with PILQR outperforms all other methods we compare against, with orders of magnitude fewer samples than DDPG and TRPO, which is the only other successful algorithm.}
    \label{fig:door_opening}
    \vspace{1pt}
\end{figure}

Finally, we compare the same methods for training neural network policies
on the door opening task, shown in Figure~\ref{fig:door_opening}. TRPO requires 20 times more samples than MDGPS with PILQR to learn a successful neural network policy. The other three methods were unable to learn a policy that opens the door despite extensive hyperparameter tuning. We provide additional simulation results in Appendix~\ref{app:sim_res}.

\subsection{Real Robot Experiments}

To evaluate our method on a real robotic platform, we use a PR2 robot (see Figure~\ref{fig:cover})
to learn the following tasks:

\noindent {\bf Hockey.} The hockey task requires using a stick to hit a puck into a goal \SI{1.4}{\meter} away. The cost function consists of two parts: the distance between the current position of the stick and a target pose that is close to the puck, and the distance between the position of the puck and the goal. The puck is tracked using a motion capture system. Although the cost provides some shaping, this task presents a significant challenge due to the difference in outcomes based on whether or not the robot actually strikes the puck, making it challenging for prior methods, as we show below.

\noindent {\bf Power plug plugging.} In this task, the robot must plug a power plug into an outlet. The cost function is the distance between the plug and a target location inside the outlet. This task requires fine manipulation to fully insert the plug.
Our TVLG policies consist of 100 time steps and we control our robot at a frequency of 20 Hz.
For further details of the experimental setup, including the cost functions, we refer the reader to Appendix~\ref{app:real_setup}. 

\begin{figure}[t]
\centering
\includegraphics[width=0.99\columnwidth]{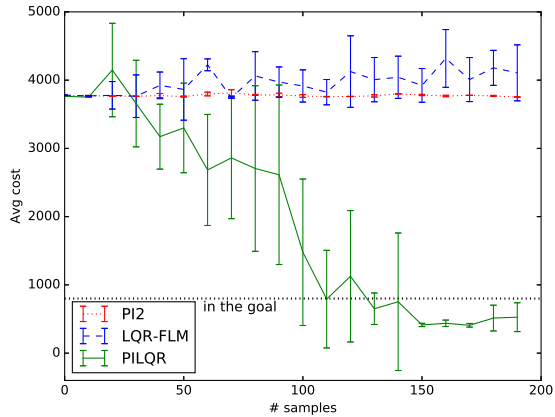}
\vspace{-8pt}
\caption{Single condition comparison of the hockey task performed on the real robot. Costs lower than the dotted line correspond to the puck entering the goal.}
\label{fig:hockey-single}
\vspace{1pt}
\end{figure}

Both of these tasks have difficult, discontinuous dynamics at the contacts between the objects, and both require a high degree of precision to succeed.
In contrast to prior works~\cite{daniel2013learning} that use kinesthetic teaching to initialize a policy that is then finetuned with model-free methods, our method does not require any human demonstrations. 
The policies are randomly initialized using a Gaussian distribution with zero mean. Such initialization does not provide any information about the task to be performed.
In all of the real robot experiments, policies are updated every 10 rollouts and the final policy is obtained after 20-25 iterations,
which corresponds to mastering the skill with less than one hour of experience.

In the first set of experiments, we aim to learn a policy that is able to hit the puck into the goal for a single position of the goal and the puck. 
The results of this experiment are shown in Figure~\ref{fig:hockey-single}.
In the case of the prior PI$^2$ method~\cite{TheodorouBS10}, the robot was not able to hit the puck. Since the puck position has the largest influence on the cost,
the resulting learning curve shows little change in the cost over the course of training.
The policy to move the arm towards the recorded arm position that enables hitting the puck turned out to be too challenging for PI$^2$ in the limited number of trials used for this experiment.
In the case of LQR-FLM, the robot was able to occasionally hit the puck in different directions. 
However, the resulting policy could not capture the complex dynamics of the sliding puck or the discrete transition, and was unable to hit the puck toward the goal.
The PILQR method was able to learn a robust policy that consistently hits the puck into the goal. 
Using the step adjustment rule described in Section~\ref{sec:step-adj}, the algorithm would shift towards model-free updates from the PI$^2$ method as the TVLG approximation of the dynamics became less accurate.
Using our method, the robot was able to get to the final position of the arm using fast model-based updates from LQR-FLM
and learn the puck-hitting policy, which is difficult to model, by automatically shifting towards model-free PI$^2$ updates.

\begin{figure}
\centering
\includegraphics[trim=4pt 10pt 0pt 2pt, clip=true, width=0.99\columnwidth]{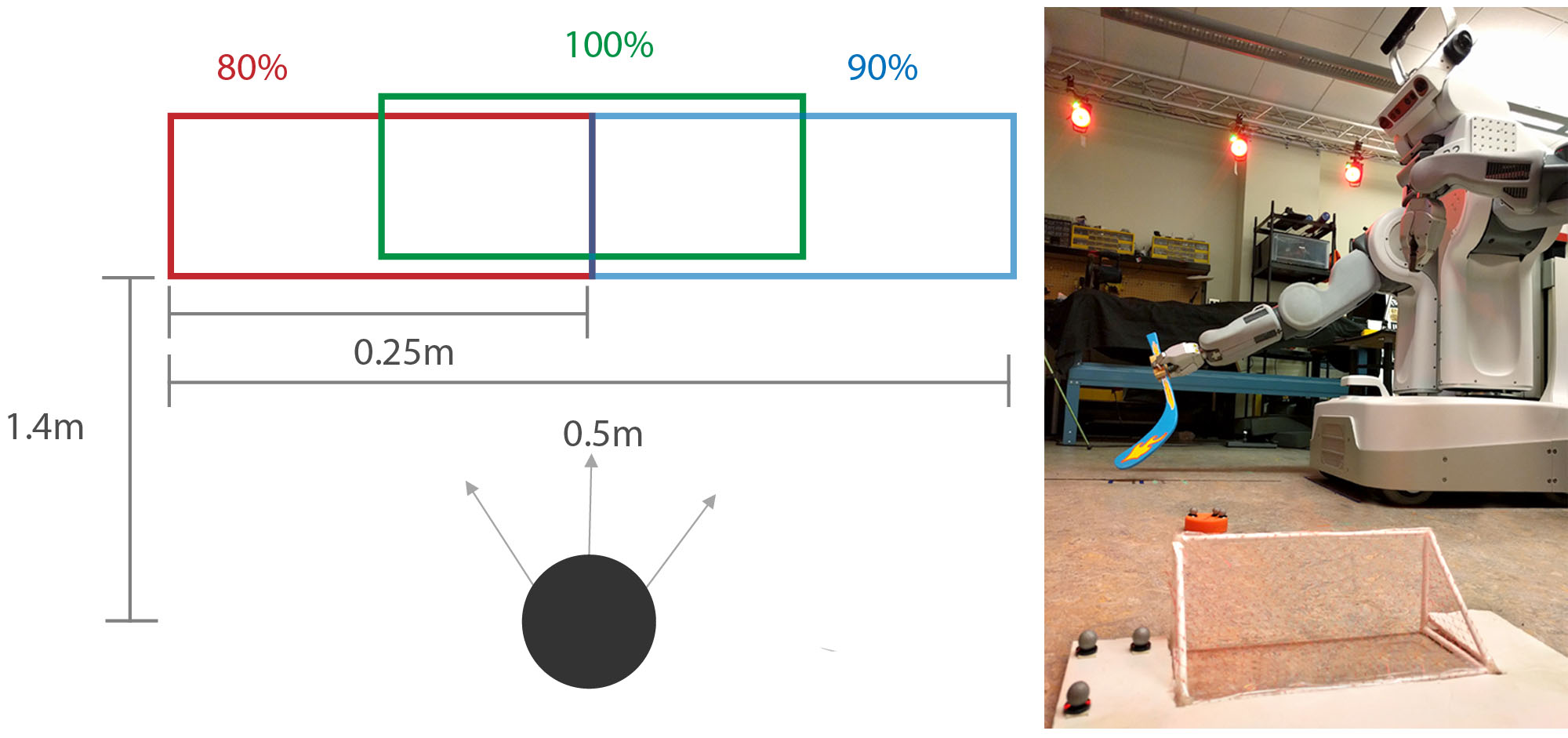}
\vspace{-10pt}
\caption{Experimental setup of the hockey task and the success rate of the final PILQR-MDGPS policy. \textit{Red} and \textit{Blue}: goal positions used for training, \textit{Green}: new goal position.}
\label{fig:hockey-diagram}
\end{figure}

In our second set of hockey experiments, we evaluate whether we can learn a neural network policy using the MDGPS-PILQR algorithm that can hit the puck into different goal locations. The goals were spaced \SI{0.5}{\meter} apart (see Figure~\ref{fig:hockey-diagram}).
The strategies for hitting the puck into different goal positions differ substantially, since the robot must adjust the arm pose to approach the puck from the right direction and aim toward the target. This makes it quite challenging to learn a single policy for this task.
We performed 30 rollouts for three different positions of the goal (10 rollouts each), two of which were used during training.
The neural network policy was able to hit the puck into the goal in 90\% of the cases (see Figure~\ref{fig:hockey-diagram}).
This shows that our method can learn high-dimensional neural network policies that generalize across various conditions. 

The results of the plug experiment are shown in Figure~\ref{fig:socket-single}.
PI$^2$ alone was unable to reach the socket. The LQR-FLM algorithm succeeded only 60\% of the time at convergence.
In contrast to the peg insertion-style tasks evaluated in prior work that used LQR-FLM~\cite{lwa-lnnpg-15}, this task requires very fine manipulation due to the small size of the plug.
Our method was able to converge to a policy that plugged in the power plug on every rollout at convergence.
The supplementary video illustrates the final behaviors of each method for both the hockey and power plug tasks.\footnote{\mbox{\url{https://sites.google.com/site/icml17pilqr}}}

\section{Discussion and Future Work}

\begin{figure}
\centering
\includegraphics[trim=0pt 0pt 0pt 38pt, clip=true, width=0.99\columnwidth]{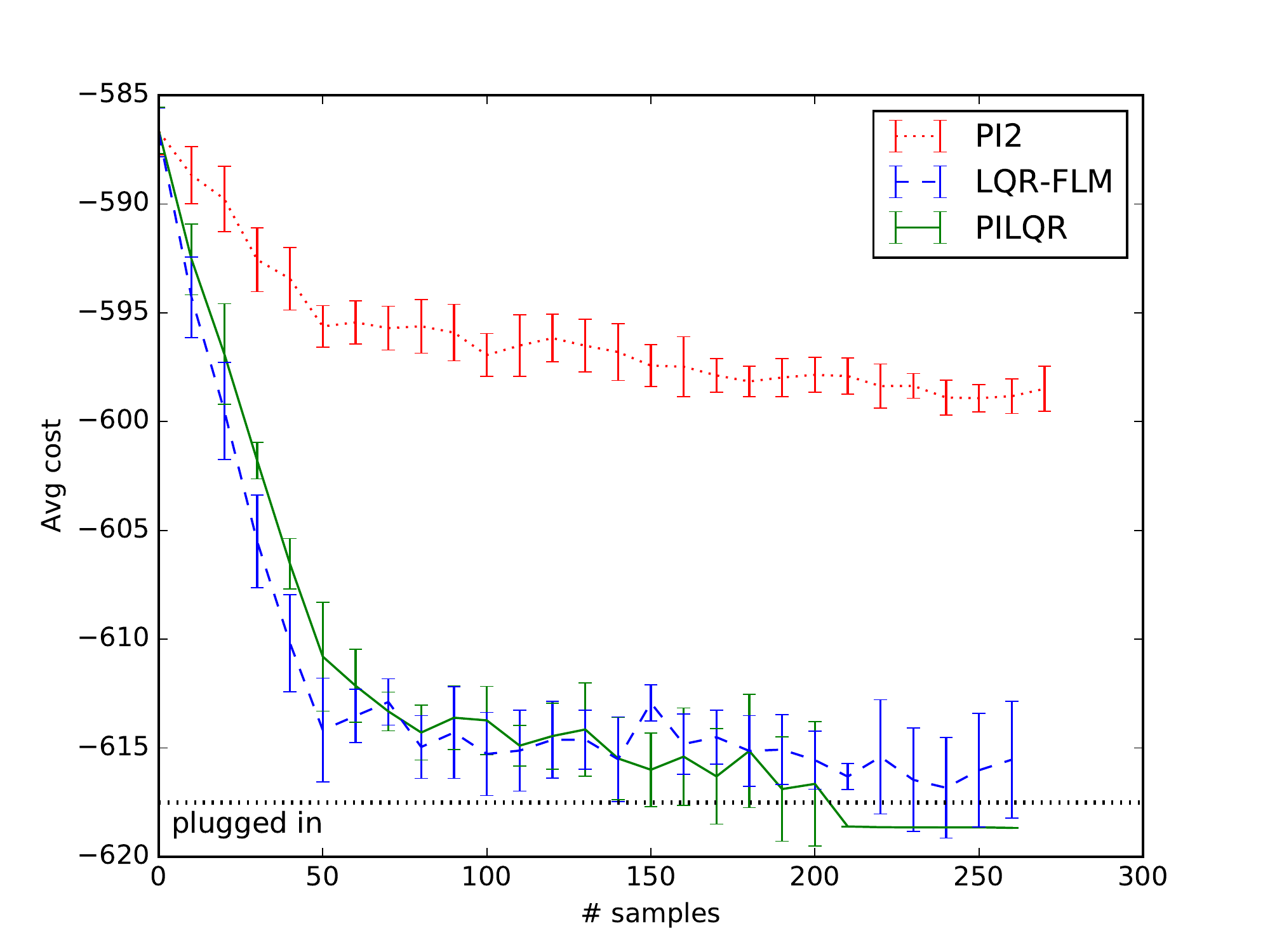}
\vspace{-13pt}
\caption{Single condition comparison of the power plug task performed on the real robot. Note that costs above the dotted line correspond to executions that did not actually insert the plug into the socket. Only our method (PILQR) was able to consistently insert the plug all the way into the socket by the final iteration.}
\label{fig:socket-single}
\end{figure}

We presented an algorithm that combines elements of model-free and model-based RL, with the aim of combining the sample efficiency of model-based methods with the ability of model-free methods to improve the policy even in situations where the model's structural assumptions are violated. We show that a particular choice of policy representation -- TVLG controllers -- is amenable to fast optimization with model-based LQR-FLM and model-free PI$^2$ algorithms using sample-based updates. We propose a hybrid algorithm based on these two components, where the PI$^2$ update is performed on the residuals between the true sample-based cost and the cost estimated under the local linear models. This algorithm has a number of appealing properties: it naturally trades off between model-based and model-free updates based on the amount of model error, can easily be extended with a KL-divergence constraint for stable learning, and can be effectively used for real-world robotic learning. We further demonstrate that, although this algorithm is specific to TVLG policies, it can be integrated into the GPS framework in order to train arbitrary parameterized policies, including deep neural networks.

We evaluated our approach on a range of challenging simulated and real-world tasks. The results show that our method combines the efficiency of model-based learning with the ability of model-free methods to succeed on tasks with discontinuous dynamics and costs. We further illustrate in direct comparisons against state-of-the-art model-free deep RL methods that, when combined with the GPS framework, our method achieves substantially better sample efficiency. It is worth noting, however, that the application of trajectory-centric RL methods such as ours, even when combined with GPS, requires the ability to reset the environment into consistent initial states~\cite{LevineA14,Levine:2016}. Recent work proposes a clustering method for lifting this restriction by sampling trajectories from random initial states and assembling them into task instances after the fact~\cite{montgomery_ajay_icra_paper}. Integrating this technique into our method would further improve its generality. An additional limitation of our method is that the form of both the model-based and model-free update requires a continuous action space. Extensions to discrete or hybrid action spaces would require some kind of continuous relaxation, and this is left for future work.

\section*{Acknowledgements} 

The authors would like to thank Sean Mason for his help with preparing the real robot experiments.
This work was supported in part by National Science Foundation grants IIS-1614653, IIS-1205249, IIS-1017134, EECS-0926052, the Office of Naval Research, the Okawa Foundation, and the Max-Planck-Society. Marvin Zhang was supported by a BAIR fellowship. Any opinions, findings, and conclusions or recommendations expressed in this material are those of the authors and do not necessarily reflect the views of the funding organizations.

\setlength{\bibsep}{2.85pt}
\bibliography{references}

\begin{thebibliography}{30}
\providecommand{\natexlab}[1]{#1}
\providecommand{\url}[1]{\texttt{#1}}
\expandafter\ifx\csname urlstyle\endcsname\relax
  \providecommand{\doi}[1]{doi: #1}\else
  \providecommand{\doi}{doi: \begingroup \urlstyle{rm}\Url}\fi

\bibitem[Akrour et~al.(2016)Akrour, Abdolmaleki, Abdulsamad, and
  Neumann]{peters_quadratic_models_paper}
Akrour, R., Abdolmaleki, A., Abdulsamad, H., and Neumann, G.
\newblock Model-free trajectory optimization for reinforcement learning.
\newblock In \emph{ICML}, 2016.

\bibitem[Brockman et~al.(2016)Brockman, Cheung, Pettersson, Schneider,
  Schulman, Tang, and Zaremba]{gym}
Brockman, G., Cheung, V., Pettersson, L., Schneider, J., Schulman, J., Tang,
  J., and Zaremba, W.
\newblock {OpenAI} gym.
\newblock \emph{arXiv preprint arXiv:1606.01540}, 2016.

\bibitem[Chebotar et~al.(2017)Chebotar, Kalakrishnan, Yahya, Li, Schaal, and
  Levine]{chebotar-icra2017}
Chebotar, Y., Kalakrishnan, M., Yahya, A., Li, A., Schaal, S., and Levine, S.
\newblock Path integral guided policy search.
\newblock In \emph{ICRA}, 2017.

\bibitem[Daniel et~al.(2013)Daniel, Neumann, Kroemer, and
  Peters]{daniel2013learning}
Daniel, Christian, Neumann, Gerhard, Kroemer, Oliver, and Peters, Jan.
\newblock Learning sequential motor tasks.
\newblock In \emph{ICRA}, 2013.

\bibitem[Deisenroth et~al.(2011)Deisenroth, Rasmussen, and Fox]{DeisenrothRF11}
Deisenroth, M., Rasmussen, C., and Fox, D.
\newblock Learning to control a low-cost manipulator using data-efficient
  reinforcement learning.
\newblock In \emph{RSS}, 2011.

\bibitem[Deisenroth et~al.(2013)Deisenroth, Neumann, and Peters]{policysearch}
Deisenroth, M., Neumann, G., and Peters, J.
\newblock A survey on policy search for robotics.
\newblock \emph{Foundations and Trends in Robotics}, 2\penalty0 (1-2):\penalty0
  1--142, 2013.

\bibitem[Deisenroth et~al.(2014)Deisenroth, Fox, and Rasmussen]{pilco}
Deisenroth, M., Fox, D., and Rasmussen, C.
\newblock Gaussian processes for data-efficient learning in robotics and
  control.
\newblock \emph{PAMI}, 2014.

\bibitem[Farshidian et~al.(2014)Farshidian, Neunert, and
  Buchli]{Farshidianetal}
Farshidian, F., Neunert, M., and Buchli, J.
\newblock Learning of closed-loop motion control.
\newblock In \emph{IROS}, 2014.

\bibitem[Gu et~al.(2016)Gu, Lillicrap, Sutskever, and Levine]{GuLSL16}
Gu, S., Lillicrap, T., Sutskever, I., and Levine, S.
\newblock Continuous deep {Q}-learning with model-based acceleration.
\newblock \emph{CoRR}, abs/1603.00748, 2016.

\bibitem[Heess et~al.(2015)Heess, Wayne, Silver, Lillicrap, Tassa, and
  Erez]{stochastic_value_gradients}
Heess, N., Wayne, G., Silver, D., Lillicrap, T., Tassa, Y., and Erez, T.
\newblock Learning continuous control policies by stochastic value gradients.
\newblock In \emph{NIPS}, 2015.

\bibitem[Kingma \& Welling(2013)Kingma and Welling]{KingmaW13}
Kingma, D. and Welling, M.
\newblock Auto-encoding variational {B}ayes.
\newblock \emph{CoRR}, abs/1312.6114, 2013.

\bibitem[Kober et~al.(2013)Kober, Bagnell, and Peters]{kbp-rlrs-13}
Kober, J., Bagnell, J., and Peters, J.
\newblock Reinforcement learning in robotics: a survey.
\newblock \emph{International Journal of Robotic Research}, 32\penalty0
  (11):\penalty0 1238--1274, 2013.

\bibitem[Levine \& Abbeel(2014)Levine and Abbeel]{LevineA14}
Levine, S. and Abbeel, P.
\newblock Learning neural network policies with guided policy search under
  unknown dynamics.
\newblock In \emph{NIPS}, 2014.

\bibitem[Levine et~al.(2015)Levine, Wagener, and Abbeel]{lwa-lnnpg-15}
Levine, S., Wagener, N., and Abbeel, P.
\newblock Learning contact-rich manipulation skills with guided policy search.
\newblock In \emph{ICRA}, 2015.

\bibitem[Levine et~al.(2016)Levine, Finn, Darrell, and Abbeel]{Levine:2016}
Levine, S., Finn, C., Darrell, T., and Abbeel, P.
\newblock End-to-end training of deep visuomotor policies.
\newblock \emph{JMLR}, 17\penalty0 (1), 2016.

\bibitem[Lillicrap et~al.(2016)Lillicrap, Hunt, Pritzel, Heess, Erez, Tassa,
  Silver, and Wierstra]{lhphe-ccdrl-16}
Lillicrap, T., Hunt, J., Pritzel, A., Heess, N., Erez, T., Tassa, Y., Silver,
  D., and Wierstra, D.
\newblock Continuous control with deep reinforcement learning.
\newblock In \emph{ICLR}, 2016.

\bibitem[Lioutikov et~al.(2014)Lioutikov, Paraschos, Neumann, and
  Peters]{peters_linear_gaussian_controller_paper}
Lioutikov, R., Paraschos, A., Neumann, G., and Peters, J.
\newblock Sample-based information-theoretic stochastic optimal control.
\newblock In \emph{ICRA}, 2014.

\bibitem[Mnih et~al.(2013)Mnih, Kavukcuoglu, Silver, Graves, Antonoglou,
  Wierstra, and Riedmiller]{mnih_et_al_atari}
Mnih, V., Kavukcuoglu, K., Silver, D., Graves, A., Antonoglou, I., Wierstra,
  D., and Riedmiller, M.
\newblock Playing {A}tari with deep reinforcement learning.
\newblock In \emph{NIPS Workshop on Deep Learning}, 2013.

\bibitem[Montgomery \& Levine(2016)Montgomery and Levine]{MontgomeryL16}
Montgomery, W. and Levine, S.
\newblock Guided policy search via approximate mirror descent.
\newblock In \emph{NIPS}, 2016.

\bibitem[Montgomery et~al.(2017)Montgomery, Ajay, Finn, Abbeel, and
  Levine]{montgomery_ajay_icra_paper}
Montgomery, W., Ajay, A., Finn, C., Abbeel, P., and Levine, S.
\newblock Reset-free guided policy search: efficient deep reinforcement
  learning with stochastic initial states.
\newblock In \emph{ICRA}, 2017.

\bibitem[Pan \& Theodorou(2014)Pan and Theodorou]{pddp}
Pan, Y. and Theodorou, E.
\newblock Probabilistic differential dynamic programming.
\newblock In \emph{NIPS}, 2014.

\bibitem[Pastor et~al.(2009)Pastor, Hoffmann, Asfour, and
  Schaal]{peter_pastor_demonstration}
Pastor, P., Hoffmann, H., Asfour, T., and Schaal, S.
\newblock Learning and generalization of motor skills by learning from
  demonstration.
\newblock In \emph{ICRA}, 2009.

\bibitem[Peters \& Schaal(2008)Peters and
  Schaal]{peters_schaal_2008_reinforcement_learning_pg}
Peters, J. and Schaal, S.
\newblock Reinforcement learning of motor skills with policy gradients.
\newblock \emph{Neural Networks}, 21\penalty0 (4), 2008.

\bibitem[Peters et~al.(2010)Peters, M\"ulling, and Altun]{PetersMA10}
Peters, J., M\"ulling, K., and Altun, Y.
\newblock Relative entropy policy search.
\newblock In \emph{AAAI}, 2010.

\bibitem[Schaal et~al.(2003)Schaal, Peters, Nakanishi, and Ijspeert]{rl_dmps}
Schaal, S., Peters, J., Nakanishi, J., and Ijspeert, A.
\newblock Control, planning, learning, and imitation with dynamic movement
  primitives.
\newblock In \emph{IROS Workshop on Bilateral Paradigms on Humans and
  Humanoids}, 2003.

\bibitem[Schulman et~al.(2015)Schulman, Levine, Moritz, Jordan, and
  Abbeel]{slmja-trpo-15}
Schulman, J., Levine, S., Moritz, P., Jordan, M., and Abbeel, P.
\newblock Trust region policy optimization.
\newblock In \emph{ICML}, 2015.

\bibitem[Schulman et~al.(2016)Schulman, Moritz, Levine, Jordan, and
  Abbeel]{trpo-gae}
Schulman, J., Moritz, P., Levine, S., Jordan, M., and Abbeel, P.
\newblock High-dimensional continuous control using generalized advantage
  estimation.
\newblock In \emph{ICLR}, 2016.

\bibitem[Sutton(1990)]{dyna-q}
Sutton, R.
\newblock Integrated architectures for learning, planning, and reacting based
  on approximating dynamic programming.
\newblock In \emph{ICML}, 1990.

\bibitem[Tassa et~al.(2012)Tassa, Erez, and Todorov]{synthesis}
Tassa, Y., Erez, T., and Todorov, E.
\newblock Synthesis and stabilization of complex behaviors.
\newblock In \emph{IROS}, 2012.

\bibitem[Theodorou et~al.(2010)Theodorou, Buchli, and Schaal]{TheodorouBS10}
Theodorou, E., Buchli, J., and Schaal, S.
\newblock A generalized path integral control approach to reinforcement
  learning.
\newblock \emph{JMLR}, 11, 2010.

\end{thebibliography}
\bibliographystyle{icml2017}

\clearpage

\section{Appendix}

\subsection{Derivation of LQR-FLM}
\label{app:lqr_flm}

Given a TVLG dynamics model and quadratic cost approximation, we can approximate our Q and value functions to second order with the following dynamic programming updates, which proceed from the last time step $t = T$ to the first step $t = 1$:
\begin{align*}
    Q_{\mathbf{x},t}=c_{\mathbf{x},t}&+\mathbf{f}_{\mathbf{x},t}^\top{V}_{\mathbf{x},t+1}\,,~~Q_{\mathbf{xx},t}=c_{\mathbf{xx},t}+\mathbf{f}_{\mathbf{x},t}^\top{V}_{\mathbf{xx},t+1}\mathbf{f}_{\mathbf{x},t}\,,\\
    Q_{\mathbf{u},t}=c_{\mathbf{u},t}&+\mathbf{f}_{\mathbf{u},t}^\top{V}_{\mathbf{x},t+1}\,,~~Q_{\mathbf{uu},t}=c_{\mathbf{uu},t}+\mathbf{f}_{\mathbf{u},t}^\top{V}_{\mathbf{xx},t+1}\mathbf{f}_{\mathbf{u},t}\,,\\
    &Q_{\mathbf{xu},t}=c_{\mathbf{xu},t}+\mathbf{f}_{\mathbf{x},t}^\top{V}_{\mathbf{xx},t+1}\mathbf{f}_{\mathbf{u},t}\,,\\
    &V_{\mathbf{x},t}=Q_{\mathbf{x},t}-Q_{\mathbf{xu},t}Q_{\mathbf{uu},t}^{-1}Q_{\mathbf{u},t}\,,\\
    &V_{\mathbf{xx},t}=Q_{\mathbf{xx},t}-Q_{\mathbf{xu},t}Q_{\mathbf{uu},t}^{-1}Q_{\mathbf{ux},t}\,.
\end{align*}
Here, similar to prior work, we use subscripts to denote derivatives. It can be shown (e.g., in~\cite{synthesis}) that the action $\mathbf{u}_t$ that minimizes the second-order approximation of the Q-function at every time step $t$ is given by
\[
\mathbf{u}_t=-Q_{\mathbf{uu},t}^{-1}Q_{\mathbf{ux},t}\mathbf{x}_t-Q_{\mathbf{uu},t}^{-1}Q_{\mathbf{u},t}\,.
\]
This action is a linear function of the state $\mathbf{x}_t$, thus we can construct an optimal linear policy by setting $\mathbf{K}_t=-Q_{\mathbf{uu},t}^{-1}Q_{\mathbf{ux},t}$ and $\mathbf{k}_t=-Q_{\mathbf{uu},t}^{-1}Q_{\mathbf{u},t}$. We can also show that the maximum-entropy policy that minimizes the approximate Q-function is given by
\[
p(\mathbf{u}_t|\mathbf{x}_t)=\mathcal{N}(\mathbf{K}_t\mathbf{x}_t+\mathbf{k}_t,Q_{\mathbf{uu},t}).
\]
This form is useful for LQR-FLM, as we use intermediate policies to generate samples to fit TVLG dynamics. \citet{LevineA14} impose a constraint on the total KL-divergence between the old and new trajectory distributions induced by the policies through an augmented cost function $\bar{c}(\mathbf{x}_t,\mathbf{u}_t)=\frac{1}{\eta}c(\mathbf{x}_t,\mathbf{u}_t)-\log{p}^{(i-1)}(\mathbf{u}_t|\mathbf{x}_t)$, where solving for $\eta$ via dual gradient descent can yield an exact solution to a KL-constrained LQR problem, where there is a single constraint that operates at the level of trajectory distributions $p(\tau)$. We can instead impose a separate KL-divergence constraint at each time step with the constrained optimization
\begin{align*}
    \min_{\mathbf{u}_t,\mathbf{\Sigma}_t}~~&\mathbb{E}_{\mathbf{x}\sim{p}(\mathbf{x}_t),\mathbf{u}\sim\mathcal{N}(\mathbf{u}_t,\mathbf{\Sigma}_t)}[Q(\mathbf{x},\mathbf{u})]\\
                                s.t.~~&\mathbb{E}_{\mathbf{x}\sim{p}(\mathbf{x}_t)}[D_{KL}(\mathcal{N}(\mathbf{u}_t,\mathbf{\Sigma}_t)\|p^{(i-1)})]\leq\epsilon_t\,.
\end{align*}
The new policy will be centered around $\mathbf{u}_t$ with covariance term $\mathbf{\Sigma}_t$. Let the old policy be parameterized by $\bar{\mathbf{K}}_t$, $\bar{\mathbf{k}}_t$, and $\bar{\mathbf{C}}_t$. We form the Lagrangian (dividing by $\eta_t$), approximate $Q$, and expand the KL-divergence term to get
\begin{align*}
    &\mathcal{L}(\mathbf{u}_t,\mathbf{\Sigma}_t,\eta_t)\\
    &=\frac{1}{\eta_t}\left[Q_{\mathbf{x},t}^\top\mathbf{x}_t+Q_{\mathbf{u},t}^\top\mathbf{u}_t+\frac{1}{2}\mathbf{x}_t^\top{Q}_{\mathbf{xx},t}\mathbf{x}_t+\frac{1}{2}tr(Q_{\mathbf{xx},t}\Sigma_{\mathbf{x},t})\right.\\
    &~~~~~~~~~~~~\left.+\frac{1}{2}\mathbf{u}_t^\top{Q}_{\mathbf{uu},t}\mathbf{u}_t+\frac{1}{2}tr(Q_{\mathbf{uu},t}\mathbf{\Sigma}_t)+\mathbf{x}_t^\top{Q}_{\mathbf{xu},t}\mathbf{u}_t\right]\\
                                                 &~~~+\frac{1}{2}\left[\log|\bar{\mathbf{\Sigma}}_t|-\log|\mathbf{\Sigma}_t|-d+tr(\bar{\mathbf{\Sigma}}_t^{-1}\mathbf{\Sigma}_t)\right.\\
                                                 &~~~~~~~~~~~~\left.+(\bar{\mathbf{K}}_t\mathbf{x}_t+\bar{\mathbf{k}}_t-\mathbf{u}_t)^\top\bar{\mathbf{\Sigma}}_t^{-1}(\bar{\mathbf{K}}_t\mathbf{x}_t+\bar{\mathbf{k}}_t-\mathbf{u}_t)\right.\\
                                                 &~~~~~~~~~~~~\left.+tr(\bar{\mathbf{K}}_t^\top\bar{\mathbf{\Sigma}}_t^{-1}\bar{\mathbf{K}}_t\Sigma_{\mathbf{x},t})\right]-\epsilon_t\,.
\end{align*}
Now we set the derivative of $\mathcal{L}$ with respect to $\mathbf{\Sigma}_t$ equal to $0$ and get
\begin{align*}
\mathbf{\Sigma}_t=\left(\frac{1}{\eta_t}Q_{\mathbf{uu},t}+\bar{\mathbf{\Sigma}}_t^{-1}\right)^{-1}\,.
\end{align*}

Setting the derivative with respect to $\mathbf{u}_t$ equal to $0$, we get
\begin{align*}
\mathbf{u}_t=-\mathbf{\Sigma}_t\left(\frac{1}{\eta_t}Q_{\mathbf{u},t}+\frac{1}{\eta_t}Q_{\mathbf{ux},t}\mathbf{x}_t-\hat{\mathbf{C}}_t^{-1}(\hat{\mathbf{K}}_t\mathbf{x}_t+\hat{\mathbf{k}}_t)\right)\,,
\end{align*}
Thus our updated mean has the parameters
\begin{align*}
      \mathbf{k}_t&=-\mathbf{\Sigma}_t\left(\frac{1}{\eta_t}Q_{\mathbf{u},t}-\hat{\mathbf{C}}_t^{-1}\hat{\mathbf{k}}_t\right)\,,\\
      \mathbf{K}_t&=-\mathbf{\Sigma}_t\left(\frac{1}{\eta_t}Q_{\mathbf{ux},t}-\hat{\mathbf{C}}_t^{-1}\hat{\mathbf{K}}_t\right)\,.
\end{align*}
As discussed by~\citet{synthesis}, when the updated $\mathbf{K}_t$ and $\mathbf{k}_t$ are not actually the optimal solution for the current quadratic Q-function, the update to the value function is a bit more complex, and is given by
\begin{align*}
     V_{\mathbf{x},t}&=Q_{\mathbf{x},t}^\top+Q_{\mathbf{u},t}^\top\mathbf{K}_t+\mathbf{k}_t^\top{Q}_{\mathbf{uu},t}\mathbf{K}_t+\mathbf{k}_t^\top{Q}_{\mathbf{ux},t}\,,\\
    V_{\mathbf{xx},t}&=Q_{\mathbf{xx},t}+\mathbf{K}_t^\top{Q}_{\mathbf{uu},t}\mathbf{K}_t+2Q_{\mathbf{xu},t}\mathbf{K}_t\,.
\end{align*}

\subsection{PI$^2$ update through constrained optimization}
\label{app:pi2_proof}
The structure of the proof for the PI$^2$ update follows~\cite{PetersMA10}, applied to the cost-to-go $S(\mathbf{x}_t, \mathbf{u}_t)$. Let us first consider the cost-to-go $S(\mathbf{x}_t, \mathbf{u}_t)$ of a single trajectory or path $ (\mathbf{x}_t, \mathbf{u}_t,\mathbf{x}_{t+1}, \mathbf{u}_{t+1}, \dots,  \mathbf{x}_T, \mathbf{u}_T)$ where $T$ is the maximum number of time steps. We can rewrite the Lagrangian in a sample-based form as
\begin{align*}
    &\mathcal{L}(p^{(i)}, \eta_t) = \\
    & \sum 
    \left(p^{(i)}S(\mathbf{x}_{t},\mathbf{u}_{t})\right) + \eta_t \left( \sum p^{(i)}\log \frac{p^{(i)}}{p^{(i-1)}}- \epsilon \right)\,.
\end{align*}
Taking the derivative of $\mathcal{L}(p^{(i)}, \eta_t)$ with respect to a single optimal policy $p^{(i)}$ and setting it to zero results in 
\begin{align*}
    \frac{\partial \mathcal{L}}{\partial p^{(i)}} &= S(\mathbf{x}_{t},\mathbf{u}_{t}) + \eta_t \left(\log \frac{p^{(i)}}{p^{(i-1)}}  + p^{(i)} \frac{p^{(i-1)}}{p^{(i)}} \frac{1}{p^{(i-1)}}\right) \\
    &= S(\mathbf{x}_{t},\mathbf{u}_{t}) + \eta_t \log \frac{p^{(i)}}{p^{(i-1)}} = 0\,.
\end{align*}
Solve the derivative for $p^{(i)}$ by exponentiating both sides
\begin{align*}
    \log \frac{p^{(i)}}{p^{(i-1)}} = -\frac{1}{\eta_t} S(\mathbf{x}_{t},\mathbf{u}_{t})\,,\\
    p^{(i)} = p^{(i-1)} \exp\left(-\frac{1}{\eta_t}S(\mathbf{x}_{t},\mathbf{u}_{t})\right)\,.
\end{align*}
This gives us a probability update rule for a single sample that only considers cost-to-go of one path. However, when sampling from a stochastic policy $p^{(i-1)}\left (  \mathbf{u}_t  | \mathbf{x}_t \right)$, there are multiple paths that start in state $\mathbf{x}_t$ with action $\mathbf{u}_t$ and continue with a noisy policy afterwards. Hence, the updated policy $p^{(i)}\left (  \mathbf{u}_t  | \mathbf{x}_t \right)$  will incorporate all of these paths as
\begin{align*}
\!\!p^{(i)}\!\left(  \mathbf{u}_t  | \mathbf{x}_t \right) 
\!\propto\! p^{(i-1)}\!\left(\mathbf{u}_t| \mathbf{x}_t \right) \!\mathbb{E}_{p^{(i-1)}}\!\!\left[\exp \left(\!-\frac{1}{\eta_t} S(\mathbf{x}_{t},\mathbf{u}_{t}) \right)\!\right]\,.
\end{align*}
The updated policy is additionally subject to normalization, which corresponds to computing the normalized probabilities in Eq.~(\ref{eqn:pi2update}).

\subsection{Detailed Experimental Setup}

\subsubsection{Simulation Experiments}
\label{app:sim_setup}

All of our cost functions use the following generic loss term on a vector $\mathbf{z}$
\begin{align}
\ell(\mathbf{z})=\frac{1}{2}\alpha\|\mathbf{z}\|_2^2+\beta\sqrt{\gamma+\|\mathbf{z}\|_2^2}\,.
\label{eq:cost}
\end{align}
$\alpha$ and $\beta$ are hyperparameters that weight the squared $\ell_2$ loss and Huber-style loss, respectively, and we set $\gamma=10^{-5}$.

On the gripper pusher task, we have three such terms. The first sets $\mathbf{z}$ as the vector difference between the block and goal positions with $\alpha=10$ and $\beta=0.1$. $\mathbf{z}$ for the second measures the vector difference between the gripper and block positions, again with $\alpha=10$ and $\beta=0.1$, and the last loss term penalizes the magnitude of the fourth robot joint angle with $\alpha=10$ and $\beta=0$. We include this last term because, while the gripper moves in 3D, the block is constrained to a 2D plane and we thus want to encourage the gripper to also stay in this plane. These loss terms are weighted by $4$, $1$, and $1$ respectively.

On the reacher task, our only loss term uses as $\mathbf{z}$ the vector difference between the arm end effector and the target, with $\alpha=0$ and $\beta=1$. For both the reacher and door opening tasks, we also include a small torque penalty term that penalizes unnecessary actuation and is typically several orders of magnitude smaller than the other loss terms.

On the door opening task, we use two loss terms. For the first, $\mathbf{z}$ measures the difference between the angle in radians of the door hinge and the desired angle of $-1.0$, with $\alpha=1$ and $\beta=0$. The second term is time-varying: for the first 25 time steps, $\mathbf{z}$ is the vector difference between the bottom of the robot end effector and a position above the door handle, and for the remaining time steps, $\mathbf{z}$ is the vector difference from the end effector to inside the handle. This encourages the policy to first navigate to a position above the handle, and then reach down with the hook to grasp the handle. Because we want to emphasize the second loss during the beginning of the trajectory and gradually switch to the first loss, we do a time-varying weighting between the loss terms. The weight of the second loss term is fixed to $1$, but the weight of the first loss term at time step $t$ is $5\left(\frac{t}{T}\right)^2$.

For the neural network policy architectures, we use two fully-connected hidden layers of rectified linear units (ReLUs) with no output non-linearity. On the reacher task, the hidden layer size is 32 units per layer, and on the door opening task, the hidden layer size is 100 units per layer.

\begin{figure}
    \centering
    \setlength{\unitlength}{0.5\columnwidth}
    \begin{picture}(1.0, 1.5)
        \put(-0.5, 0.75){\includegraphics[width=0.49\columnwidth]{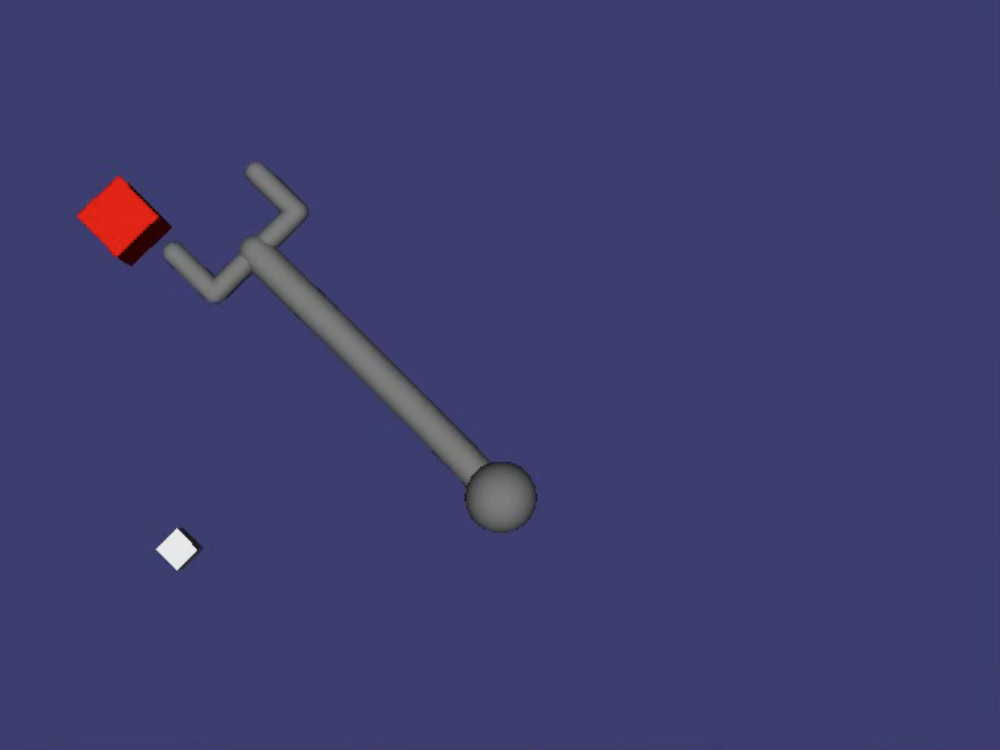}}
        \put(0.5, 0.75){\includegraphics[width=0.49\columnwidth]{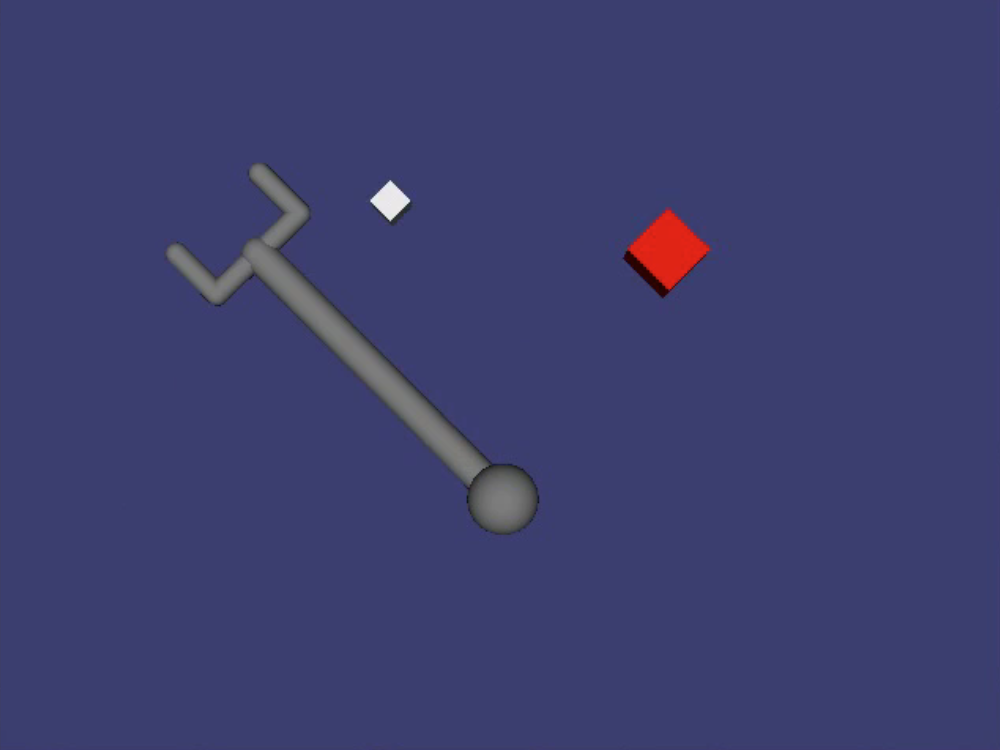}}
        \put(-0.5, 0.0){\includegraphics[width=0.49\columnwidth]{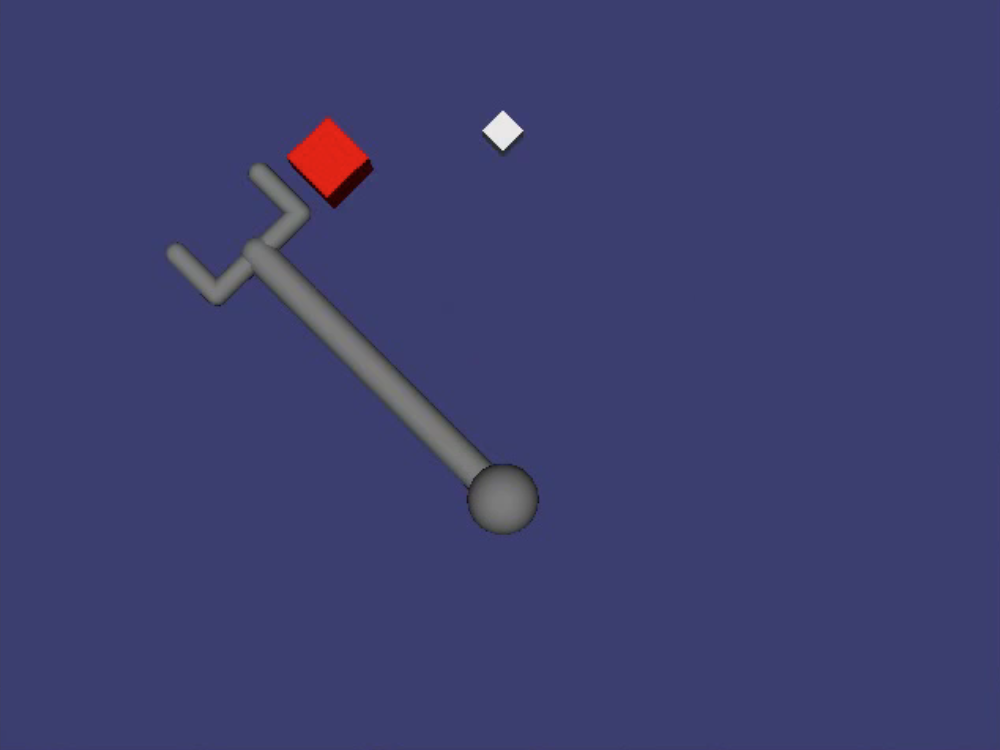}}
        \put(0.5, 0.0){\includegraphics[width=0.49\columnwidth]{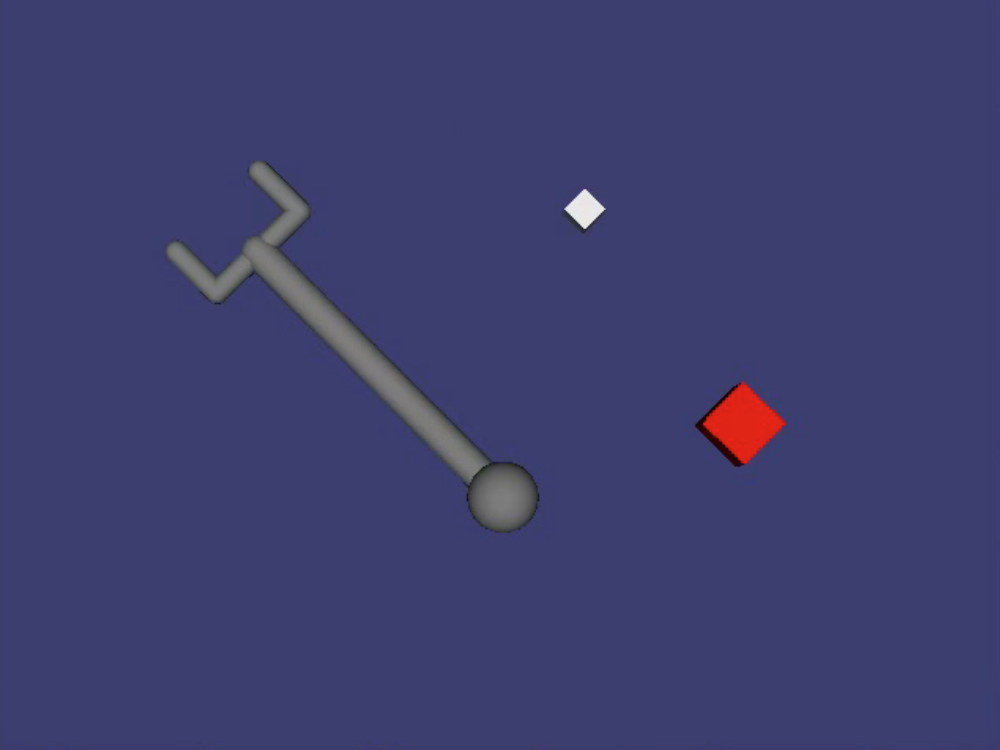}}
    \end{picture}
    \caption{The initial conditions for the gripper pusher task that we train TVLG policies on. The top left and bottom right conditions are more difficult due to the distance from the block to the goal and the configuration of the arm. The top left condition results are reported in Section~\ref{sec:simresults}.}
    \label{fig:gripper-conds}
\end{figure}

All of the tasks involve varying conditions for which we train one TVLG policy per condition and, for reacher and door opening, train a neural network policy to generalize across all conditions. For gripper pusher, the conditions vary the starting positions of the block and the goal, which can have a drastic effect on the difficulty of the task. Figure~\ref{fig:gripper-conds} illustrates the four initial conditions of the gripper pusher task for which we train TVLG policies. For reacher, analogous to Open\-AI Gym, we vary the initial arm configuration and position of the target and train TVLG policies from 16 randomly chosen conditions. Note that, while Open\-AI Gym randomizes this initialization per episode, we always reset to the same condition when training TVLG policies as this is an additional assumption we impose. However, when we test the performance of the neural network policy, we collect 300 test episodes with random initial conditions. For the door opening task, we initialize the robot position within a small square in the ground plane. We train TVLG policies from the four corners of this square and test our neural network policies with 100 test episodes from random positions within the square.

For the gripper pusher and door opening tasks, we train TVLG policies with PILQR, LQR-FLM and PI$^2$ with 20 episodes per iteration per condition for 20 iterations. In Appendix~\ref{app:sim_res}, we also test PI$^2$ with 200 episodes per iteration. For the reacher task, we use 3 episodes per iteration per condition for 10 iterations. Note that we do not collect any additional samples for training neural network policies. For the prior methods, we train DDPG with 150 episodes per epoch for 80 epochs on the reacher task, and TRPO uses 600 episodes per iteration for 120 iterations. On door opening, TRPO uses 400 episodes per iteration for 80 iterations and DDPG uses 160 episodes per epoch for 100 epochs, though note that DDPG is ultimately not successful.

\subsubsection{Real Robot Experiments}
\label{app:real_setup}

For the real robot tasks we use a hybrid cost function that includes two loss terms of the form of Eq.~\ref{eq:cost}. The first loss term $\ell_{arm}(\mathbf{z})$ computes the difference between the current position of the robot's end-effector and the position of the end-effector when the hockey stick is located just in front of the puck. We set $\alpha = 0.1$ and $\beta = 0.0001$ for this cost function. The second loss term $\ell_{goal}(\mathbf{z})$ is based on the distance between the puck and the goal that we estimate using a motion capture system. We set $\alpha = 0.0$ and $\beta = 1.0$. Both $\ell_{arm}$ and $\ell_{goal}$ have a linear ramp, which makes the cost increase towards the end of the trajectory. In addition, We include a small torque cost term $\ell_{torque}$ to penalize unnecessary high torques. The combined function sums over all the cost terms: $\ell_{total} = 100.0 \, \ell_{goal} + \ell_{arm} + \ell_{torque}$. We give a substantially higher weight to the cost on the distance to the goal to achieve a higher precision of the task execution.

Our neural network policy includes two fully-connected hidden layers of rectified linear units (ReLUs). Each of the hidden layers consists of 42 neurons. The inputs of the policy include: puck and goal positions measured with a motion capture system, robot joint angles, joint velocities, the end-effector pose and the end-effector velocity. During PILQR-MDGPS training, we use data augmentation to regularize the neural network. In particular, the observations were augmented with Gaussian noise to mitigate overfitting to the noisy sensor data. 

\subsection{Additional Simulation Results}
\label{app:sim_res}

\begin{figure}
\centering
\includegraphics[width=0.99\columnwidth]{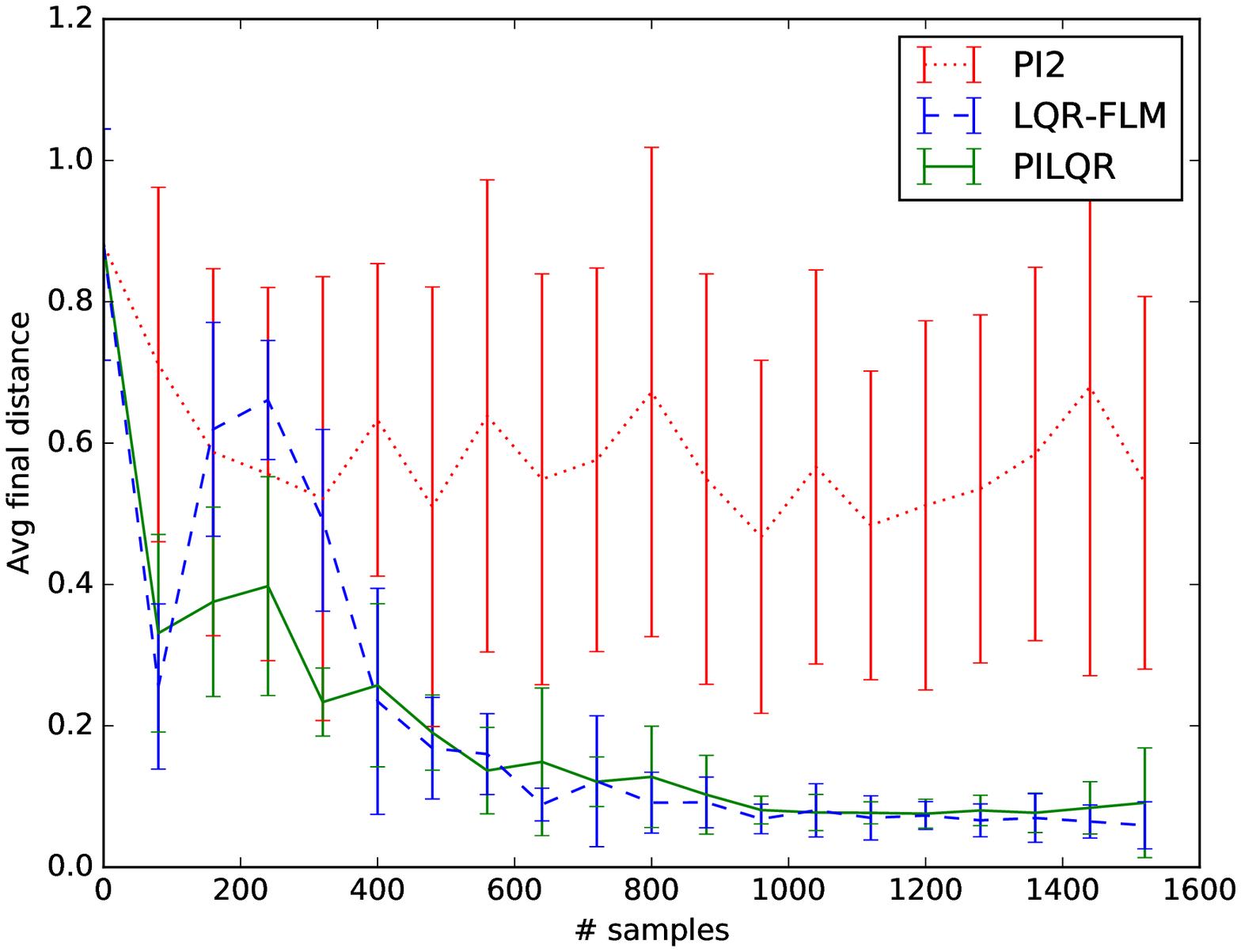}\\
\includegraphics[width=0.99\columnwidth]{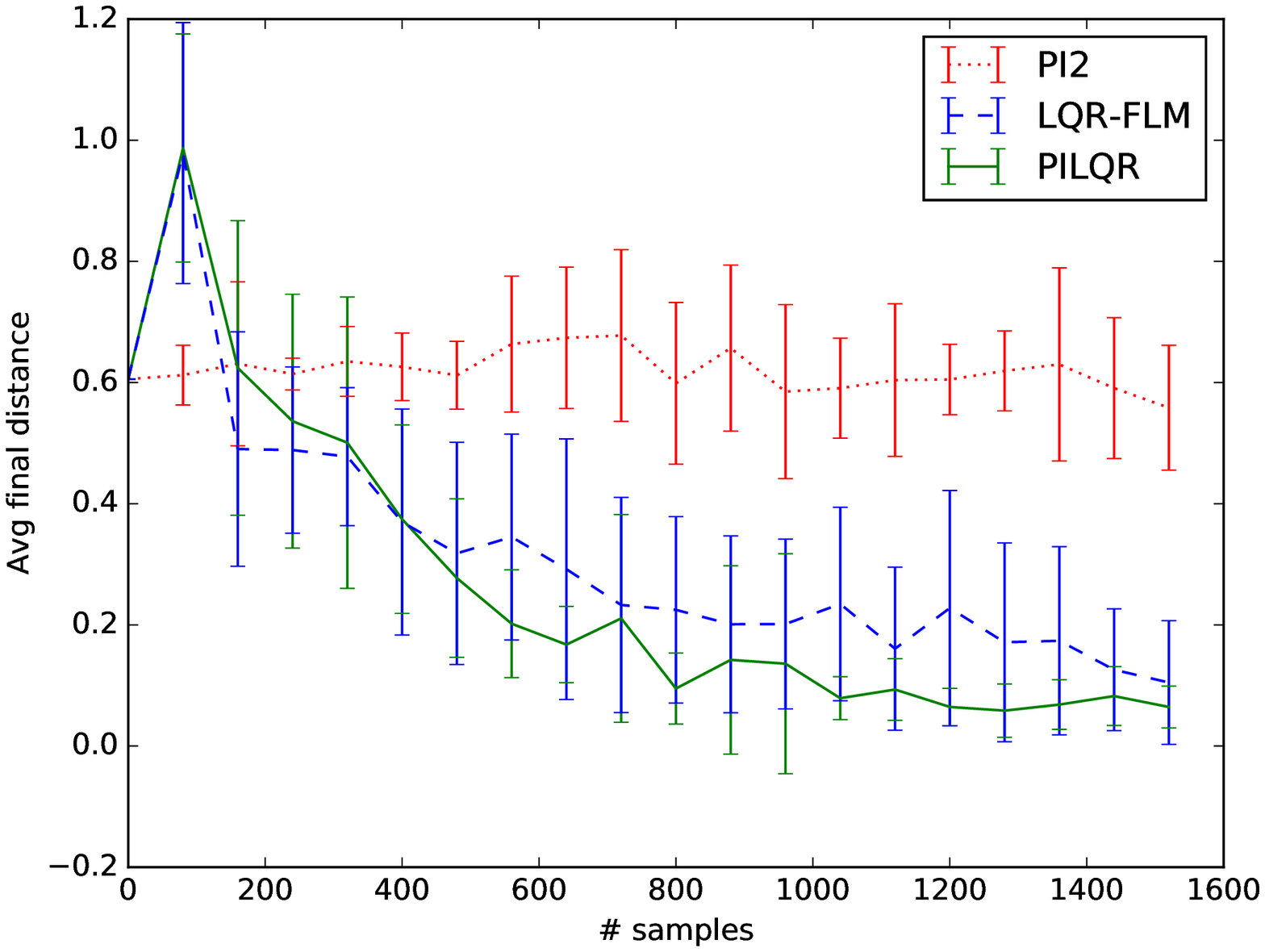}\\
\includegraphics[width=0.99\columnwidth]{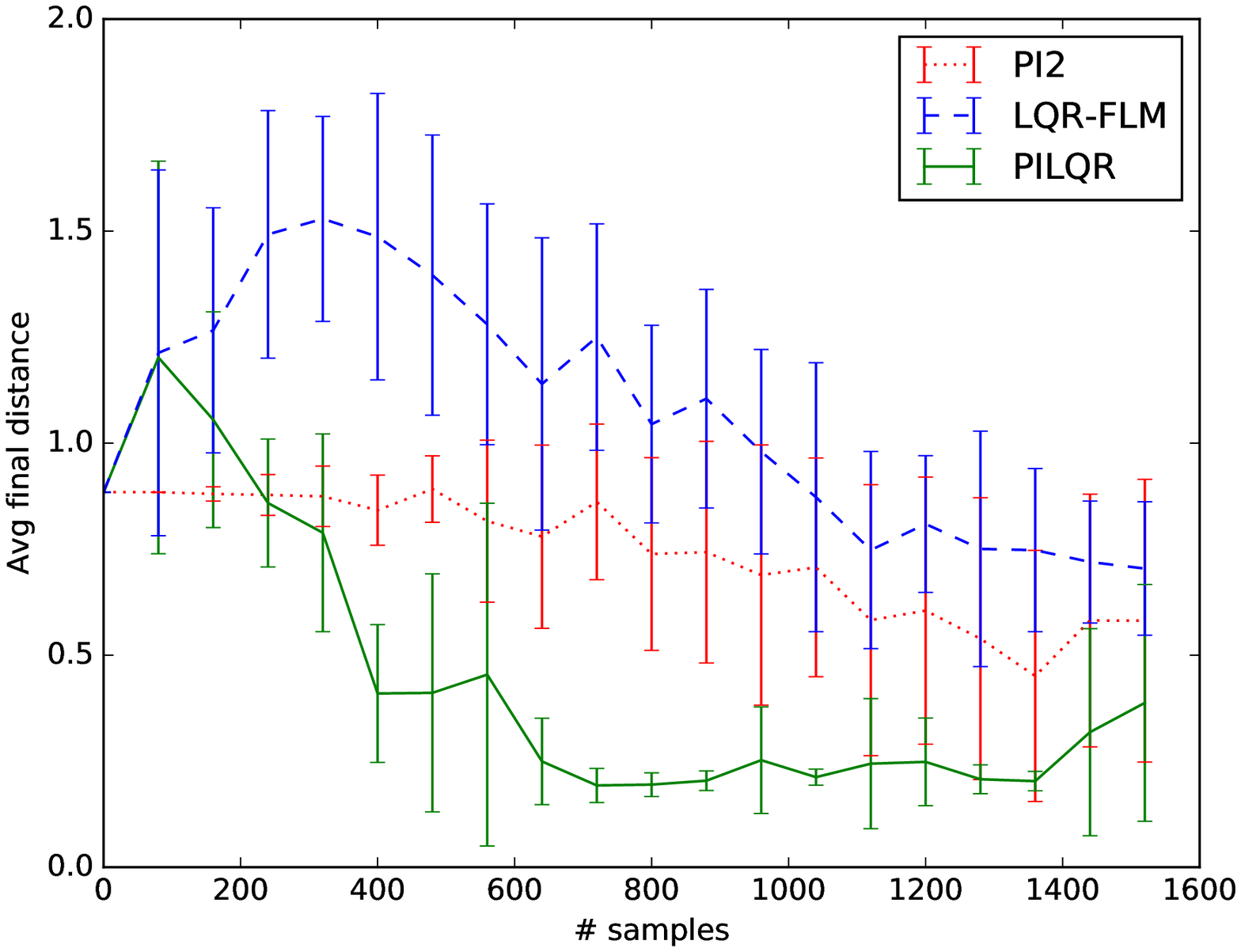}\\
\caption{Single condition comparisons of the gripper-pusher task in three additional conditions. The top, middle, and bottom plots correspond to the top right, bottom right, and bottom left conditions depicted in Figure~\ref{fig:gripper-conds}, respectively. The PILQR method outperforms other baselines in two out of the three conditions. The conditions presented in the top and middle figure are significantly easier than the other conditions presented in the paper.}
\label{fig:additional-gripper}
\end{figure}

Figure~\ref{fig:additional-gripper} shows additional simulation results obtained for the gripper pusher task for the three additional initial conditions.
The instances presented here are not as challenging as the one reported in the paper. 
Our method (PILQR) is able to outperform other baselines except for the first two conditions presented in the first rows of Figure~\ref{fig:additional-gripper}, where LQR-FLM performs equally well due to the simplicity of the task. PI$^2$ is not able to make progress with the same number of samples, however, its performance on each condition is comparable to LQR-FLM when provided with 10 times more samples. 

We also test PI$^2$ with 10 times more samples on the reacher and door opening tasks. On the reacher task, PI$^2$ improves substantially with more samples, though it is still worse than the four other methods. However, as Figure~\ref{fig:more-door} shows, PI$^2$ is unable to succeed on the door opening task even with 10 times more samples. The performance of PI$^2$ is likely to continue increasing as we provide even more samples.

\begin{figure}
\centering
\includegraphics[width=0.99\columnwidth]{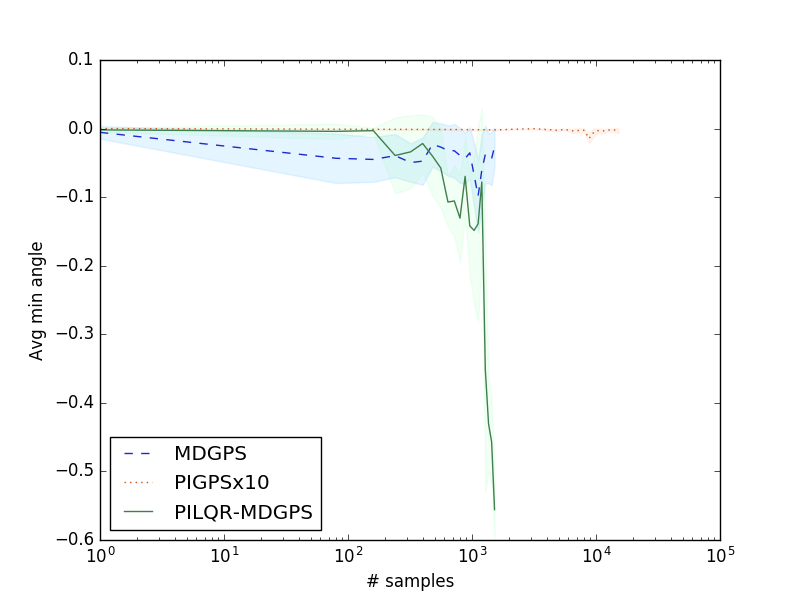}
\caption{Additional results on the door opening task.}
\label{fig:more-door}
\end{figure}

\end{document}